\documentclass[sigconf]{aamas}  

\usepackage{booktabs}
\usepackage{graphicx}
\usepackage{amsmath}
\usepackage{algorithm}
\usepackage{algorithmic}
\usepackage{booktabs}
\usepackage{color}
\usepackage{bbm}
\usepackage[titletoc]{appendix}
\usepackage{blindtext}
\usepackage{verbatim}
\usepackage{stfloats}
\usepackage{color}
\usepackage[font=small,labelfont=bf]{caption}
\usepackage[english]{babel}
\usepackage{amsthm}
\usepackage{subfigure}
\usepackage{url}

\newcommand{\states}{\mathcal{S}}
\newcommand{\rlstate}{s}
\newcommand{\actions}{\mathcal{A}}
\newcommand{\action}{a}
\newcommand{\rewards}{\mathcal{R}}
\newcommand{\reward}{r}
\newcommand{\discount}{\gamma}
\newcommand{\transitions}{\mathcal{P}}

\DeclareMathOperator*{\expectation}{\mathbb{E}}

\AtBeginDocument{%
  \providecommand\BibTeX{{%
    \normalfont B\kern-0.5em{\scshape i\kern-0.25em b}\kern-0.8em\TeX}}}
    
\usepackage{booktabs}    
\usepackage{flushend} 

\setcopyright{ifaamas}  
\copyrightyear{2020} 
\acmYear{2020} 
\acmDOI{} 
\acmPrice{} 
\acmISBN{} 
\acmConference[AAMAS'20]{Proc.\@ of the 19th International Conference on Autonomous Agents and Multiagent Systems (AAMAS 2020)}{May 9--13, 2020}{Auckland, New Zealand}{B.~An, N.~Yorke-Smith, A.~El~Fallah~Seghrouchni, G.~Sukthankar (eds.)}  

\begin{document}

\title{Off-Policy Deep Reinforcement Learning with Analogous Disentangled Exploration}  




\author{Anji Liu}
\affiliation{
    \institution{University of California, Los Angeles}
    \streetaddress{404 Westwood Plaza}
}
\email{anjiliu219@gmail.com}

\author{Yitao Liang}
\affiliation{
    \institution{University of California, Los Angeles}
    \streetaddress{404 Westwood Plaza}
}
\email{yliang@cs.ucla.edu}

\author{Guy Van den Broeck}
\affiliation{
    \institution{University of California, Los Angeles}
    \streetaddress{404 Westwood Plaza}
}
\email{guyvdb@cs.ucla.edu}

\renewcommand{\shortauthors}{A. Liu, Y. Liang, and G.V.D. Broeck}

\newlength{\dbltextfloatsepsave} \setlength{\dbltextfloatsepsave}{\dbltextfloatsep}

\begin{abstract}  
Off-policy reinforcement learning (RL) is concerned with learning a rewarding policy by executing another policy that gathers samples of experience. While the former policy (i.e. target policy) is rewarding but in-expressive (in most cases, deterministic), doing well in the latter task, in contrast, requires an expressive policy (i.e. behavior policy) that offers guided and effective exploration.
Contrary to most methods that make a trade-off between optimality and expressiveness, disentangled frameworks explicitly decouple the two objectives, which each is dealt with by a distinct separate policy. 
Although being able to freely design and optimize the two policies with respect to their own objectives, naively disentangling them can lead to inefficient learning or stability issues. To mitigate this problem, our proposed method \emph{Analogous Disentangled Actor-Critic} (ADAC) designs analogous pairs of actors and critics. Specifically, ADAC leverages a key property about Stein variational gradient descent (SVGD) to constraint the expressive energy-based behavior policy with respect to the target one for effective exploration. Additionally, an analogous critic pair is introduced to incorporate intrinsic rewards in a principled manner, with theoretical guarantees on the overall learning stability and effectiveness.
We empirically evaluate environment-reward-only ADAC on 14 continuous-control tasks and report the state-of-the-art on 10 of them. 
We further demonstrate ADAC, when paired with intrinsic rewards, outperform alternatives in exploration-challenging tasks.
\end{abstract}


\maketitle


\section{Introduction}
\label{Introduction}

Reinforcement learning (RL) studies the control problem where an agent tries to navigate through an unknown environment \cite{sutton2018reinforcement}. The agent attempts to maximize its cumulative rewards through an iterative trial-and-error learning process \cite{arulkumaran2017brief}. Recently, we have seen many successes of applying RL to challenging simulation \cite{Mnih15,Liang16} and real-world \cite{Silver17,Leibo2017,Wang2017} problems. 
Inherently, RL consists of two distinct but closely related objectives: learn the best possible policy from the gathered samples (i.e. \emph{exploitation}) and collect new samples effectively (i.e. \emph{exploration}). While the \emph{exploitation} step shares certain similarities with tasks such as supervised learning, \emph{exploration} is unique, essential, and is often viewed as the backbone of many successful RL algorithms \cite{mnih2013playing,haarnoja2018soft2}.

In order to explore novel states that are potentially rewarding, it is crucial to incorporate randomness when interacting with the environment. Thanks to its simplicity, injecting noise into the action \cite{lillicrap2015continuous,fujimoto2018addressing} or parameter space \cite{fortunato2017noisy,plappert2017parameter} is widely used to implicitly construct behavior policies from target policies. In most prior work, the injected noise has a mean of zero, such that the updates to the target policy have no~bias \cite{fujimoto2018off,gu2016q}.
The stability of noise-based exploration, which is obtained from its non-biased nature, makes it a safe exploration strategy.
However, noise-based approaches are generally less effective since they are neither aware of potentially rewarding actions nor guided by the exploration-oriented targets.

To tackle the above problem, two orthogonal lines of approaches have been proposed. One of them considers extracting more information from the current knowledge (i.e. gathered samples). For example, energy-based RL algorithms learn to capture potentially rewarding actions through its energy objective \cite{haarnoja2018soft2,sutton2018reinforcement}. 
A second line of work considers leveraging external guidance to aid exploration. In a nutshell, they formulate some intuitive tendencies in exploration as an additional reward function called intrinsic reward \cite{bellemare2016unifying,houthooft2016vime}. Guided by these auxiliary tasks, RL algorithms tend to act curiously, substantially improving exploration of the state space.

Despite their promising exploration efficiency, both lines of work fail to fully \emph{exploit} the collected samples and turn them into the highest performing policy, as their learned policy often executes sub-optimal actions. To avoid this undesirable \emph{exploration-exploitation} trade-off, several attempts have been made to separately design two policies (i.e. disentangle them), of which one aims to gather the most informative examples (and hence is commonly referred as the \emph{behavior policy}) while the other attempts to best utilize the current knowledge from the gathered samples (and hence is usually referred as the \emph{target policy}) \cite{colas2018gep,beyer2019mulex}. To help fulfill their respective goals, disentangled objective functions and learning paradigms are further designed and separately applied to the two policies.

However, naively disentangling the behavior from the target policy would render their update process unstable. For example, when disentangled naively, the two policies tend to differ substantially due to their contrasting objectives, which is known to potentially result in catastrophic learning failure \cite{nachum2018trustpcl}. To mitigate this problem, we propose \emph{Analogous Disentangled Actor-Critic} (ADAC), where being \emph{analogous} is reflected by the constraints imposed on the disentangled actor-critic \cite{mnih2016asynchronous} pairs. ADAC consists of two main algorithmic contributions. First, \emph{policy co-training} guides the behavior policy's update by the target policy, making the gathered samples more helpful for the target policy's learning process while keeping the expressiveness of the behavior policy for extensive exploration (Section~\ref{Policy Co-training}). Second, \emph{critic bounding} allows an additional explorative critic to be trained with the aid of intrinsic rewards (Section~\ref{Critic Bounding}). Under certain constraints from the target policy, the resultant critic maintains the curiosity incentivized by intrinsic rewards while guarantees training stability of the target policy.

Besides Section~\ref{Analogous Disentangled Actor-Critic}'s elaboration of our method, the rest of the paper is organized as follows. Section~\ref{Related Work} reviews and summarizes the related work. Key background concepts and notations are introduced in Section~\ref{Preliminaries}. Experiment details of ADAC are explained in Section~\ref{experiments}. Finally, conclusions are presented in Section~\ref{conclusion}.\footnote{We provide code to reproduce our experiments at  \url{https://github.com/UCLA-StarAI/Analogous-Disentangled-Actor-Critic}.}

\section{Related Work}
\label{Related Work}

Learning to be aware of potentially rewarding actions is a promising strategy to conduct exploration, as it automatically prunes less rewarding actions and concentrates exploration efforts on those with high potential. To capture these actions, expressive learning models/objectives are widely used. Most noticeable recent work on this direction, such as Soft Actor-Critic \cite{haarnoja2018soft2}, EntRL \cite{schulman2017equivalence}, and Soft Q Learning \cite{haarnoja2017reinforcement}, learns an expressive energy-based target policy according to the maximum entropy RL objective \cite{ziebart2010modeling}. However, the expressiveness of their policies in turn becomes a burden for their optimality, and in practice, trade-offs such as temperature controlling \cite{haarnoja2018soft} and reward scaling \cite{haarnoja2017reinforcement} have to be made for better overall performance. As we shall show later, ADAC makes use of a similar but extended energy-based target, and alleviates the compromise on optimality using the \emph{analogous disentangled} framework.

Ad-hoc exploration-oriented learning targets that are designed to better explore the state space are also promising. Some recent research efforts on this line include count-based exploration \cite{xu2017study,bellemare2016unifying} and intrinsic motivation \cite{houthooft2016vime,fu2017ex2,kulkarni2016hierarchical} approaches. The outcome of these methods is usually an auxiliary reward termed the \emph{intrinsic reward}, which is extremely useful when the environment-defined reward is sparsely available. However, as we shall illustrate in Section~\ref{Evaluation in Sparse-Reward Environments}, intrinsic reward potentially biases the task-defined learning objective, leading to catastrophic failure in some tasks. Again, with the disentangled nature of ADAC, we give a principled solution to solve this problem with theoretical guarantees (Section~\ref{Critic Bounding}).

Explicitly disentangling exploration from exploitation has been used to solve a common problem in the above approaches, which is sacrificing the target policy's optimality for better exploration. By separately designing exploration and exploitation components, both objectives can be better pursued simultaneously. Specifically, GEP-PG \cite{colas2018gep} uses a Goal Exploration Process (GEP) \cite{forestier2017intrinsically} to generate samples and feed them to the replay buffer of DDPG \cite{lillicrap2015continuous} or its variants. Multiple losses for exploration (MULEX) \cite{beyer2019mulex} proposes to use a series of intrinsic rewards to optimize different policies in parallel, which in turn generates abundant samples to train the target policy. Despite having intriguing conceptual ideas, they overlook the training stability issue caused by the mismatch in the distribution of collected samples (using the behavior policy) and the distribution induced by the target policy, which is formalized as \emph{extrapolation error} in \cite{fujimoto2018off}. ADAC aims to mitigate the training stability issue caused by the extrapolation error while maintaining effective exploration exploitation trade-off promised by expressive behavior policies (Section~\ref{Policy Co-training}) as well as intrinsic rewards (Section~\ref{Critic Bounding}) using its \emph{analogous disentangled} actor-critic pairs.

\footnotetext[2]{In all the environments considered in this paper, actions are assumed to be continuous.}

\section{Preliminaries}
\label{Preliminaries}

In this section, we introduce the RL setting we address in this paper, and some background concepts that we utilize to build our method.

\subsection{RL with Continuous Control}
\label{Reinforcement Learning in Continuous Action Spaces}

In a standard \emph{reinforcement learning} (RL) setup, an agent interacts with an unknown environment at discrete time steps and aims to maximize the \emph{reward} signal \cite{sutton2018reinforcement}. The environment is often formalized as a \emph{Markov Decision Process} (MDP), which can be succinctly defined as a 5-tuple $\mathcal{M} \!= <\!\states, \actions, \rewards, \transitions, \discount\!>$. At time step $t$, the agent in state $\rlstate_t \!\!\in\! \states$ takes action $\action_t \!\!\in\! \actions$ according to \emph{policy} $\pi$, a conditional distribution of $\action$ given $\rlstate$, leading to the next state $\rlstate_{t+1}$ according to the transition probability $\transitions(\rlstate_{t+1} \!\mid\! \rlstate_t, \action_t)$. Meanwhile, the agent observes reward $\reward_t \!\sim\! \rewards(\rlstate_t,\action_t)$ emitted from the environment.\footnotemark

The agent strives to learn the \emph{optimal policy} that maximizes the expected return $J(\pi)=\mathbb{E}_{s_{0} \sim \rho_{0}, \action_t \sim \pi, s_{t+1} \sim \transitions, \reward_t \sim \rewards }\left[
\sum_{t=0}^{\infty} \gamma^{t} \reward_{t} \right]$, where $\rho_{0}$ is the initial state distribution and $\discount \in [0, 1)$ is the discount factor balancing the priority of short and long-term rewards. For continuous control, the policy $\pi$ (also known as the actor in the actor-critic framework) parameterized by $\theta$ can be updated by taking the gradient $\nabla_{\theta} J(\pi)$. According to the deterministic policy gradient theorem \cite{Silver2014Deterministic}, $\nabla_{\theta} J(\pi) = \mathbb{E}_{(\rlstate, \action) \sim \rho_{\pi}}\left[\nabla_{a} Q_{\rewards}^{\pi}(s, a) \nabla_{\theta} \pi(s)\right]$, where $\rho_{\pi}$ denotes the state-action marginals of the trajectory distribution induced by $\pi$, and $Q_{\rewards}^{\pi}$ denotes the state-action value function (also know as the critic in the actor-critic framework), which represents the expected return under the reward function specified by $\rewards$ when performing action $\action$ at state $\rlstate$ and following policy $\pi$ afterwards. Intuitively, it measures how preferable executing action $\action$ is at state $\rlstate$ with respect to the policy $\pi$ and reward function $\rewards$. 
Following \cite{bellman1966dynamic}, we additionally introduce the Bellman operator, which is commonly used to update the $Q$-function. The Bellman operator $\mathcal{T}^{\pi}_{\rewards}$ uses $\rewards$ and $\pi$ to update an arbitrary value function $Q$, which is not necessarily defined with respect to the same $\pi$ or $\rewards$. For example, the outcome of $\mathcal{T}^{\pi_{1}}_{\rewards_{1}} Q^{\pi_{2}}_{\rewards_{2}}(s_{t}, a_{t})$ is defined as
$\rewards_{1}(s_{t}, a_{t}) + \gamma \mathbb{E}_{s_{t+1} \sim \transitions, a_{t+1} \sim \pi_{1}}[Q^{\pi_{2}}_{\rewards_{2}}(s_{t+1}, a_{t+1})]
$. By slightly abusing notations, we further define the outcome of $\mathcal{T}^{\mathrm{max}}_{\rewards_{1}} Q^{\pi_{2}}_{\rewards_{2}}(s_{t}, a_{t})$ as $\rewards_{1}(s_{t}, a_{t}) + \gamma \max_{a_{t+1}} \mathbb{E}_{s_{t+1} \sim \transitions}[Q^{\pi_{2}}_{\rewards_{2}}(s_{t+1}, a_{t+1})]$. Some also call $\mathcal{T}^{\mathrm{max}}_{\rewards}$ the Bellman optimality operator.

\subsection{Off-policy Learning and Behavior Policy}
\label{Off-policy Learning and Behavior Policy}

To aid exploration, it is a common practice to construct/store more than one policy for the agent (either implicitly or explicitly). Off-policy actor-critic methods \cite{Watkins92} allow us to make a clear separation between the \emph{target policy}, which refers to the best policy currently learned by the agent, and the \emph{behavior policy}, which the agent follows to interact with the environment. Note that the discussion in Section~\ref{Reinforcement Learning in Continuous Action Spaces} is largely around the target policy. Thus, starting from this point, to avoid confusion, $\pi$ is reserved to only denote the target policy and notation $\mu$ is introduced to denote the behavior policy. Due to the policy separation, the target policy $\pi$ is instead resorting to the estimates calculated with regards to samples collected by the behavior policy $\mu$, that is, the deterministic policy gradient mentioned above is approximated as
    \begin{equation}
        \label{Eq:deterministic-policy-gradient}
        \nabla_{\theta} J_{\pi} (\theta) \approx \mathbb{E}_{(\rlstate, \action) \sim \rho_{\mu}}\left[\nabla_{a} Q^{\pi}_{\rewards}(s, a) \nabla_{\theta} \pi (s) \right],
    \end{equation}
\noindent where $\rewards$ is the environment-defined reward. One of the most notable off-policy learning algorithms that capitalize on this idea is deep deterministic policy gradient (DDPG) \cite{lillicrap2015continuous}. To mitigate function approximation errors in DDPG, Fujimoto et al. proposes TD3 \cite{fujimoto2018addressing}. Given that DDPG and TD3 have demonstrated themselves to be competitive in many continuous control benchmarks, we choose to implement our \emph{Analogous Disentangled Actor Critic} (ADAC) on top of their target policies. Yet, it is worth reiterating that ADAC is compatible with any existing off-policy learning algorithms. We defer a more detailed discussion of ADAC's compatibility until we start formally introducing our method in Section~\ref{Algorithm Overview}.

\begin{figure}[t]
    \centering
    \includegraphics[width=\columnwidth]{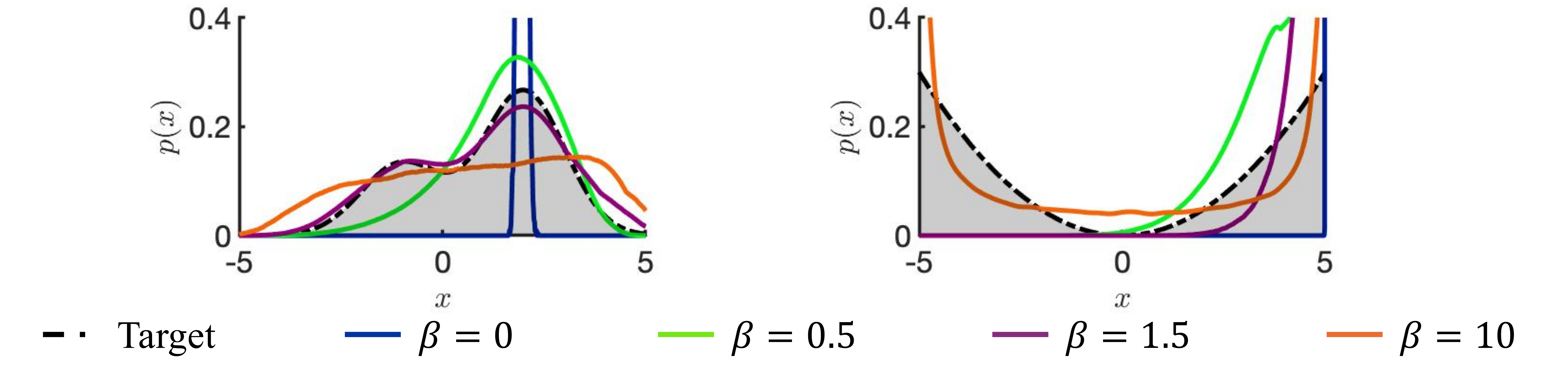}
    \caption{Evaluation of the amortized SVGD learning algorithm \cite{feng2017learning} (Eq~\eqref{eq:SVGD_gradient}) with different $\beta$ under two target distributions.}
    \label{fig:SVGD_illustration}
\end{figure}

\subsection{Expressive Behavior Policies through Energy-Based Representation}
\label{Expressive Behavior Policy by Energy-based Representations}

One promising way to design an exploration-oriented behavior policy without external guidance, which is usually in the form of intrinsic reward, is by increasing the expressiveness of $\mu$ to capture information about potentially rewarding actions.
Energy-based representations have recently been increasingly chosen as the target form to construct an expressive behavior policy. Since its first introduction by \cite{ziebart2010modeling} to achieve maximum-entropy reinforcement learning, several additional prior work keeps improving upon this idea. Among them, the most notable ones include Soft Q-Learning (SQL) \cite{haarnoja2017reinforcement}, EntRL \cite{schulman2017equivalence}, and Soft Actor-Critic (SAC) \cite{haarnoja2018soft}. Collectively, they have achieved competitive results on many benchmark tasks. Formally, the energy-based behavior policy is defined as
    \begin{equation}
        \label{eq:energy_based_behavior_policy}
        \mu(a \mid s) \propto \exp{(Q(s, a))},
    \end{equation}
\noindent where $Q$ is commonly selected to be the target critic $Q^{\pi}_{\rewards}$ in prior work \cite{haarnoja2018soft,haarnoja2018soft2}. Various efficient samplers have been proposed to approximate the distribution specified in Eq~\eqref{eq:energy_based_behavior_policy}. Among them, \cite{haarnoja2017reinforcement}'s Stein variational gradient descent (SVGD) \cite{liu2016stein,WangZ018} based sampler is especially worth noting as it has the potential to approximate complex and multi-model behavior policies. Given this, we also choose it to sample the behavior policy in our proposed ADAC.

Additionally, we want to highlight an intriguing property of SVGD that is critical for understanding why we can perform \emph{analogous disentangled exploration} effectively. Intuitively, SVGD transforms a set of particles to match a target distribution. In the context of RL, following Amortized SVGD \cite{feng2017learning}, we use a neural network sampler $f_{\varphi}(\rlstate, \xi)$ ($\xi \sim \mathcal{N}(\mathbf{0}, \mathbf{I})$) to approximate Eq~\eqref{eq:energy_based_behavior_policy}, which is done by minimizing the KL divergence between two distributions. According to \cite{feng2017learning}, $f_{\varphi}$ is updated according to the following gradient:
\begin{equation}
\label{eq:SVGD_gradient}
\begin{aligned}
    \nabla_{\varphi} J_{\mu} (\varphi) \approx \mathbb{E}_{s,\xi \sim \mathcal{N} (\mathbf{0}, \mathbf{I})}\Big[\sum_{j=1}^{K}\big[\underbrace{\mathcal{K}(a, a'_{j}) \nabla_{a'_{j}} Q(s, a'_{j})}_{\mathrm{term~1}} \\ + \beta \cdot \underbrace{\nabla_{a'_{j}} \mathcal{K}(a, a'_{j})}_{\mathrm{term~2}}\big]\big|_{a=f_{\varphi}(s, \xi)} \frac{\partial f_{\varphi}(s, \xi)}{\partial \varphi}\Big] / K,
\end{aligned}
\end{equation}
\noindent where $\mathcal{K}$ is a positive definite kernel\footnote{Formally, in ADAC, we define the kernel as $\mathcal{K}(a, \hat{a}_{i}) = \frac{1}{\sqrt{2 \pi} (d / K)} \exp \left ( - \frac{\left \| a - \hat{a}_{i} \right \|^{2}}{2 (d / K)^2} \right )$, where $d$ is the number of dimensions of the action space.}, and $\beta$ is an additional hyper-parameter proposed to make optimality-expressiveness trade-off. The intrinsic connection between Eq~\eqref{eq:SVGD_gradient} and the deterministic policy gradient (i.e. Eq~\eqref{Eq:deterministic-policy-gradient}) is introduced in \cite{haarnoja2017reinforcement} and \cite{feng2017learning}: the first term of the gradient represents a combination of deterministic policy gradients weighted by the kernel $\mathcal{K}$, while the second term of the gradient represents an entropy maximization objective.

To aid a better understanding of this relation, we illustrate the distribution approximated by SVGD using different $\beta$ in a toy example as shown in Figure~\ref{fig:SVGD_illustration}. The dashed line is the approximation target. When $\beta$ is small, the entropy of the learned distribution is restricted and the overall policy leans towards the highest-probability region. On the other hand, larger $\beta$ leads to more expressive approximation.

\section{Method}
\label{Analogous Disentangled Actor-Critic}

This section introduces our proposed method \emph{Analogous Disentangled Actor-Critic} (ADAC). We start by providing an overview of it (Section~\ref{Algorithm Overview}), which is followed by elaborating the specific choices we make to design our actors and critics (Sections~\ref{Policy Co-training} and \ref{Critic Bounding}).

\subsection{Algorithm Overview}
\label{Algorithm Overview}

Figure~\ref{fig:ADAC_overview} provides a diagram overview of ADAC, which consists of two pairs of actor-critic $\langle \mu, Q^{\pi}_{\rewards'} \rangle$ and $\langle \pi, Q^{\pi}_{\rewards} \rangle$ (see the blue and pink box) to achieve \emph{disentanglement}. Same with prior off-policy algorithms (e.g., DDPG), during training ADAC alternates between the two main procedures, namely \emph{sample collection} (dotted green box), where we use $\mu$ to interact with the environment to collect training samples, and \emph{model update} (dashed gray box), which consists of two phases: (i) batches of the collected samples are used to update both critics (the pink box); (ii) $\mu$ and $\pi$ (the blue box) are updated according to their respective critic using different objectives. During evaluation, $\pi$ is used to interact with the environment.

\footnotetext[4]{$f$ takes two components $\rlstate$ and $\xi$ as input, and $\varphi$ is the parameter set of $f$.}

Both steps in the \emph{model update} phase manifest the analogous property of our method.
First, although optimized with respect to different objectives, both $\mu$ and $\pi$ are represented by the neural network $f$, where $\mu (s) \!:=\! f_{\varphi} (s,\! \xi) |_{\xi \sim \mathcal{N} (\mathbf{0}, \mathbf{I})}$ and $\pi (s) \!:=\! f_{\varphi} (s,\! \xi) |_{\xi = [0,\dots,0]^{T}}$.\footnotemark[4] That is, $\pi$ is a deterministic policy since its input $\xi$ is fixed, while $\mu (s)$ can be regarded as an action sampler that uses the randomly sampled $\xi$ to generate actions. As we shall demonstrate in Section~\ref{Policy Co-training}, this specific setup effectively restricts the deviation between the two policies ($\mu$ and $\pi$) (i.e. update bias), which stabilizes the training process and maintains sufficient expressiveness in the behavior policy $\mu$ (also see Section~\ref{Analysis of the stochastic behavior policy} for an intuitive illustration).

\begin{figure}[t]
    \centering
    \includegraphics[width=\columnwidth]{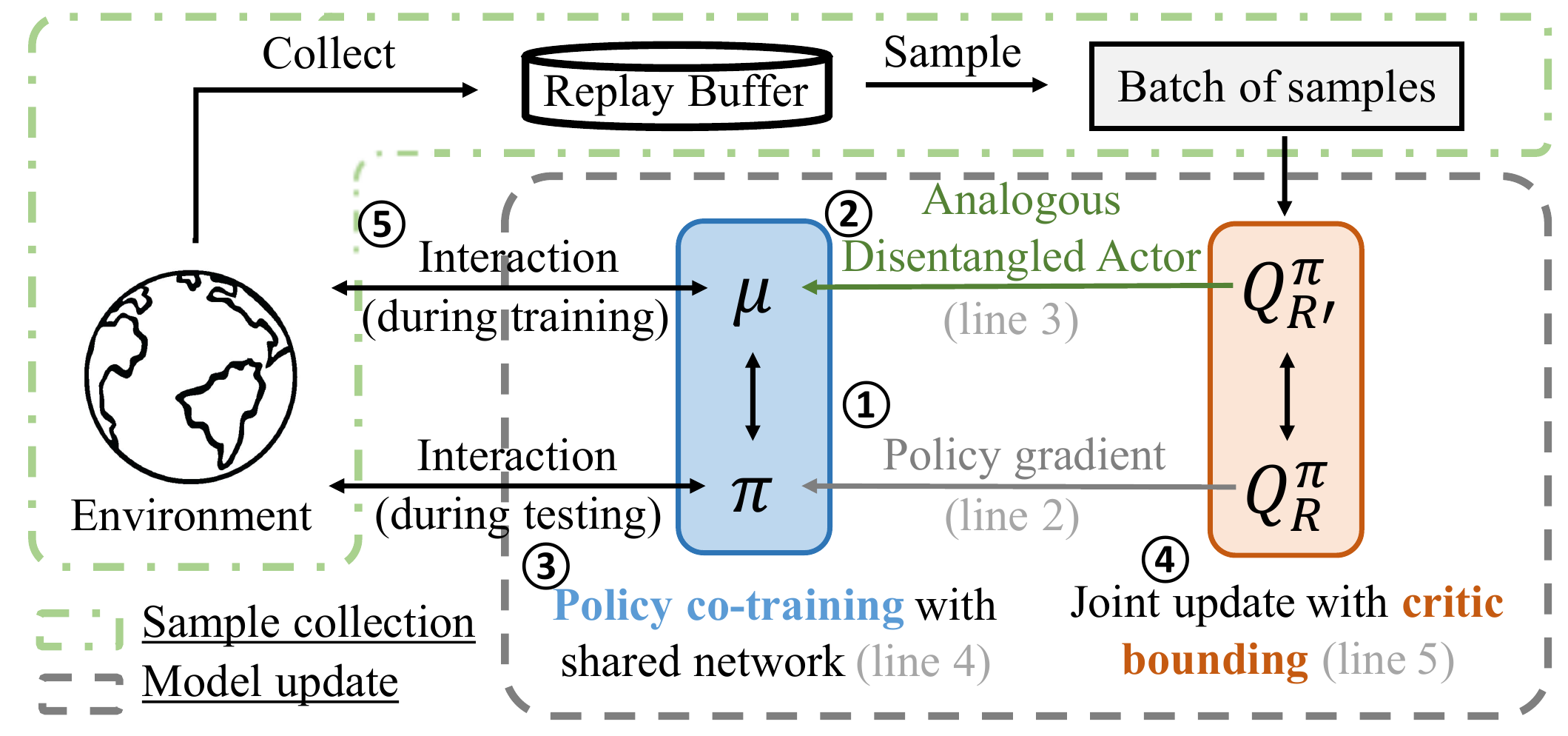}
    \caption{Block diagram of ADAC, which consists of the sample collection phase (green box with dotted line) and the model update phase (gray box with dashed line).
    Model (i.e. actor and critic networks) updates are performed
    sequentially from \textcircled{1} to \textcircled{4}. Each update step's corresponding line in Algorithm~\ref{alg:ADAC_high_level} is shown in brackets.}
    \label{fig:ADAC_overview}
\end{figure}

The second exhibit of our method's \emph{analogous} nature lies on our designed critics $Q_{\rewards}^{\pi}$ and $Q_{\rewards'}^{\pi}$, which are based on the environment-defined reward $\rewards$ and the augmented reward $\rewards' := \rewards + \rewards^{\mathrm{in}}$ ($\rewards^{\mathrm{in}}$ is the intrinsic reward) respectively yet are both computed with regard to the target policy $\pi$. As a standard approach, $Q_{\rewards}^{\pi}$ approximates the task-defined objective that the algorithm aims to maximize. On the other hand, $Q_{\rewards'}^{\pi}$ is a behavior critic that can be shown to be both \emph{explorative} and \emph{stable} theoretically (Section~\ref{Critic Bounding}) and empirically (Section~\ref{Evaluation in Sparse-Reward Environments}).
Note that when not using intrinsic reward, the two critics are degraded to be identical to one another (i.e. $\rewards = \rewards'$) and in practice when that happens we only store one of them.

To better appreciate our method, it is not enough to only gain an overview about our actors and critics in isolation. Given this, we then formalize the connections between the actors and the critics as well as the objectives that are optimized during the model update phase (Figure~\ref{fig:ADAC_overview}). As defined above, $\pi$ is the exploitation policy that aims to maintain optimality throughout the learning process, which is best optimized using the deterministic policy gradient (Eq~\eqref{Eq:deterministic-policy-gradient}), where $Q_{\rewards}^{\pi}$ is used as the referred critic (\textcircled{1} in Figure~\ref{fig:ADAC_overview}). On the other hand, for the sake of expressiveness, the energy-based objective (Eq~\eqref{eq:energy_based_behavior_policy}) is a good fit for $\mu$. To further encourage exploration, we use the behavior critic $Q_{\rewards'}^{\pi}$ in the objective, which gives $\mu (\action \mid \rlstate) \propto \exp (Q_{\rewards'}^{\pi} (\rlstate, \action))$ (\textcircled{2} in Figure~\ref{fig:ADAC_overview}). Since both policies share the same network $f$, the actor optimization process (\textcircled{3} in Figure~\ref{fig:ADAC_overview}) is done by maximizing
    \begin{align}
        J_{\pi} (\varphi) + J_{\mu} (\varphi),
        \label{eq:objective_of_f}
    \end{align}
\noindent where the gradients of both terms are defined by Eqs~\eqref{Eq:deterministic-policy-gradient} and \eqref{eq:SVGD_gradient}, respectively. In particular, we set $\pi (s) := f_{\varphi} (s, \xi) |_{\xi = [0,0,\dots,0]^{T}}$ in Eq~\eqref{Eq:deterministic-policy-gradient} and $Q := Q_{\rewards'}^{\pi}$ in Eq~\eqref{eq:SVGD_gradient}. As illustrated in Algorithm~\ref{alg:ADAC_high_level} (line 5), we update $Q^{\pi}_{\rewards}$ and $Q^{\pi}_{\rewards'}$ with the target $\mathcal{T}^{\pi} Q^{\pi}_{\rewards}$ and $\mathcal{T}^{\pi} Q^{\pi}_{\rewards'}$ on the collected samples using the mean squared error loss, respectively.

In the sample collection phase, $\mu$ interacts with the environment and the gathered samples are stored in a replay buffer \cite{mnih2013playing} for later use in the model update phase. Given state $s$, actions are sampled from $\mu$ with a three-step procedure: (i) sample $\xi \sim \mathcal{N} (\mathbf{0}, \mathbf{I})$, (ii) plug the sampled $\xi$ in $f_{\varphi} (s, \xi)$ to get its output $\hat{a}$, and (iii) regard $\hat{a}$ as the center of kernel $\mathcal{K} (\cdot, \hat{a})$\footnotemark[1] and sample an action $a$ from it.

On the implementation side, ADAC is compatible with any existing off-policy actor-critic model for continuous control: it directly builds upon them by inheriting their actor $\pi$ (which is also their target policy) and critic $Q_{\rewards}^{\pi}$. To be more specific, ADAC merely adds a new actor $\mu$ to interact with the environment and a new critic $Q_{\rewards'}^{\pi}$ that guides $\mu$'s updates on top of the base model, along with the constraints/connections enforced between the inherented and the new actor and between the inherent and the new critic (i.e. policy co-training and critic bounding). In other words, modifications made by ADAC would not conflict with the originally proposed improvements on the base model. In our experiments, two base models (i.e. DDPG \cite{lillicrap2015continuous} and TD3 \cite{fujimoto2018addressing}) are adopted.\footnotemark[5] \footnotetext[5]{See Appendix~\ref{Algorithmic Details of ADAC} for the pseudo-code and detailed description of ADAC.}

\setlength{\textfloatsep}{0.7em}
\begin{algorithm}[t]
\caption{The model update phase of ADAC. Correspondence with Figure \ref{fig:ADAC_overview} is given after ``//''.}
\label{alg:ADAC_high_level}
{\fontsize{8}{8} \selectfont
\begin{algorithmic}[1]

\STATE{\textbf{Input:} A minibatch of samples $\mathcal{B}$, actor model $f_{\varphi}$ (represents the target policy $f^{\pi}_{\varphi}$ as well as the behavior policy $f^{\mu}_{\varphi})$, critic models $Q^{\pi}_{\rewards}$ and $Q^{\pi}_{\rewards}$.}

\STATE{$\nabla_{\varphi} f^{\pi}_{\varphi}$ $\leftarrow$ the deterministic policy gradient of $Q^{\pi}_{\rewards}$ with respect to $\pi$ (Eq~\eqref{Eq:deterministic-policy-gradient}). \textcolor[RGB]{116,116,116}{\textbf{// target policy update}}}

\STATE{$\nabla_{\varphi} f^{\mu}_{\varphi}$ $\leftarrow$ gradient of $Q^{\pi}_{\rewards'}$ with respect to the behavior policy $\mu$ (Eq \eqref{eq:SVGD_gradient}, Section \ref{Expressive Behavior Policy by Energy-based Representations}) \textcolor[RGB]{74,119,47}{\textbf{// behavior policy learning}}}

\STATE{Update $f$ with $\nabla_{\varphi} f^{\pi}_{\varphi}$ and $\nabla_{\varphi} f^{\mu}_{\varphi}$ \textcolor[RGB]{40,106,173}{\textbf{// policy co-training}}}

\STATE{Update $Q^{\pi}_{\phi}$ and $Q^{\mu}_{\psi}$ to minimize the mean squared error on $\mathcal{B}$ with respect to the target $\mathcal{T}^{\pi} Q^{\pi}_{\rewards}$ and $\mathcal{T}^{\pi} Q^{\pi}_{\rewards'}$, respectively. \textcolor[RGB]{190,79,18}{\textbf{// value update with critic bounding}}}

\end{algorithmic}
}
\end{algorithm}
\setlength{\dbltextfloatsep}{\dbltextfloatsepsave}

\subsection{Stabilizing Policy Updates by Policy Co-training}
\label{Policy Co-training}

Although a behavior policy by Eq~\eqref{eq:energy_based_behavior_policy} is sufficiently expressive to capture potentially rewarding actions, it may still not be helpful for learning a better $\pi$: being expressive also means that $\mu$ is often significantly different from $\pi$, leading to collect samples that can substantially bias $\pi$'s updates (recall the discussion about Equation~\ref{Eq:deterministic-policy-gradient}), and in turn rendering the learning process of $Q^{\pi}_{\rewards}$ unstable and vulnerable to catastrophic failure \cite{sutton2008convergent,schlegel2019importance,fujimoto2018off}. To be more specific, since the difference between $\pi$ and an expressive $\mu$ is  more than some zero-mean random noise, the state marginal distribution ($\rho_{\mu}$) defined with respect to $\mu$ can potentially diverge greatly from that ($\rho_{\pi}$) defined with respect to $\pi$. Since $\rho_{\pi}$ is not directly accessible, 
as shown in Eq~\eqref{Eq:deterministic-policy-gradient}, the gradients of $\pi$ are approximated using samples from $\rho_{\mu}$. When the approximated gradients constantly deviate significantly from the true values (i.e. the approximated gradients are biased), the updates to $\pi$ essentially become inaccurate and hence ineffective. This suggests that a brutal act of disentangling the behavior policy from the target policy alone is not a guarantee of improved training efficiency or final performance.

Therefore, to mitigate the aforementioned problem, we would like to reduce the distance between $\mu$ and $\pi$, which naturally reduces the KL-divergence between distribution $\rho_{\mu}$ and $\rho_{\pi}$.
One straightforward approach to reduce the distance between the two policies is to restrict the randomness of $\mu$, for example by lowering the entropy of the behavior policy $\mu$ through a smaller $\beta$ (Eq~\eqref{eq:SVGD_gradient}). However, this inevitably sacrifices $\mu$'s expressiveness, which in turn would also harm ADAC's  competitiveness. 
Alternatively, we propose \emph{policy co-training} to best maintain the expressiveness of $\mu$ while also stabilizing it by restricting it with regards to $\pi$, which is motivated by the intrinsic connection between Eqs~\eqref{Eq:deterministic-policy-gradient} and \eqref{eq:SVGD_gradient} (see the $2nd$ paragraph of Section~\ref{Expressive Behavior Policy by Energy-based Representations}).
We reiterate here that in a nutshell, both policies are modeled by the same network $f$ and are distinguished only by their different inputs to $\xi$. During training, $f$ is updated to maximize Eq~\eqref{eq:objective_of_f}. The method to sample actions from $\mu$ is described in the $5th$ paragraph of Section~\ref{Algorithm Overview}.

We further justify the above choice by demonstrating that the imposed restrictions on $\mu$ and $\pi$ only have minor influence on $\pi$'s optimality and $\mu$'s expressiveness. To argue for this point, we need to revisit Eq~\eqref{eq:SVGD_gradient} for one more time: $\pi$ can be viewed as being updated with $\beta = 0$, whereas $\mu$ is updated with $\beta > 0$. Intuitively, this makes policy $\pi$ optimal since its action is not affected by the entropy maximization term (i.e. the second term). $\mu$ is still expressive since only when the input random variable $\xi$ is close to the zero vector, it will be significantly restricted by $\pi$. In Section~\ref{Analysis of the stochastic behavior policy},  we will empirically demonstrate policy co-training indeed reduces the distance between $\mu$ and $\pi$ during training, fulfilling its mission.

Additionally, \emph{policy co-training} enforces the underlying relations between $\pi$ and $\mu$. Specifically, policy co-training forces $\pi$ to be contained in $\mu$ since $[0, 0, \dots, 0]^{T}$ is the highest-density point of $\mathcal{N}(\mathbf{0}, \mathbf{I})$, and sampling $\xi$ from $\mathcal{N}(\mathbf{0}, \mathbf{I})$ is likely to generate actions close to that from $\pi$. This matches the intuition that $\pi$ and $\mu$ should share similarities: actions proposed by $\pi$ is rewarding (with respect to $\rewards$) and thus should be frequently executed by $\mu$.

\subsection{Incorporating Intrinsic Reward in Behavior Critic via Critic Bounding}
\label{Critic Bounding}

With the help of disentanglement as well as \emph{policy co-training}, we manage to design an expressive behavior policy that not only explores effectively but also helps stabilize $\pi$'s learning process. In this subsection, we aim to achieve the same goal -- stability and expressiveness -- on a different subject, the behavior critic $Q^{\pi}_{\rewards'}$.

As introduced in Section~\ref{Algorithm Overview}, $\rewards$ is the environment-defined reward function, while $\rewards'$ consists of an additional exploration-oriented intrinsic reward $\rewards^{in}$. As hinted by the notations, ADAC's target critic $Q_{\rewards}^{\pi}$ and behavior critic $Q_{\rewards'}^{\pi}$ are defined with regard to the same policy but updated differently according to the following
    \begin{equation}
        \label{Eq:critic_objective}
        Q_{\rewards}^{\pi} \leftarrow \mathcal{T}^{\pi}_{\rewards} Q_{\rewards}^{\pi}; \quad Q_{\rewards'}^{\pi} \leftarrow \mathcal{T}^{\pi}_{\rewards'} Q_{\rewards'}^{\pi},
    \end{equation}
\noindent where updates are performed through minibatches in practice. Note that when no intrinsic reward is used, Eq~\eqref{Eq:critic_objective} becomes trivial and the two critics ($Q_{\rewards}^{\pi}$ and $Q_{\rewards'}^{\pi}$) are identical. Therefore, we only consider the case where intrinsic reward exists in the following discussion.

While it is natural that the target critic is updated using the target policy, it may seem counterintuitive that the behavior critic is also updated using the target policy. Given that $\mu$ is updated following the guidance (i.e. through the energy-based objective) of $Q_{\rewards'}^{\pi}$, we do so to prevent $\mu$ from diverging disastrously from $\pi$.
\begin{theorem}
\label{theorem:critic_theorem}
Let $\pi$ be a greedy policy w.r.t. $Q^{\pi}_{\rewards}$ and $\mu$ be a greedy policy w.r.t. $Q^{\pi}_{\rewards'}$. Assume $Q^{\pi}_{\rewards'}$ is optimal w.r.t. $\mathcal{T}^{\pi}_{\rewards'}$ and $\rewards' (s, a) \geq \rewards (s, a) \; (\forall s, a \in \states \times \actions)$. We~~have~~the~~following~~results.

First, \noindent $\mathbb{E}_{\rho_{\pi}} [\mathcal{T}^{max}_{\rewards} Q^{\pi}_{\rewards} - Q^{\pi}_{\rewards}]$, a proxy of training stability, is lower bounded by 
    \begin{align}
        \mathbb{E}_{\rho_{\mu}} [\mathcal{T}^{max}_{\rewards} Q^{\pi}_{\rewards} - Q^{\pi}_{\rewards}] + \mathbb{E}_{\rho_{\pi}} [\rewards] - \mathbb{E}_{\rho_{\mu}} [\rewards]. \label{eq: first theoretical result}
    \end{align}
    
Second, $\mathbb{E}_{\rho_{\mu}} [\mathcal{T}^{max}_{\rewards} Q^{\pi}_{\rewards} - Q^{\pi}_{\rewards}]$, a proxy of training effectiveness, is lower bounded by 
    \begin{align}
        \mathbb{E}_{\rho_{\pi}} [\mathcal{T}^{max}_{\rewards} Q^{\pi}_{\rewards} - Q^{\pi}_{\rewards}] + \mathbb{E}_{\rho_{\pi}} [\rewards - \rewards']. \label{eq: second theoretical result}
    \end{align}
\end{theorem}

We first examine its assumptions. While others are generally satisfiable and are commonly made in the RL literature \cite{munos2007performance}, the assumption on the rewards --- $\forall s, \!a \!\in\! \states \times \actions, \rewards'(s, \!a) \!\geq\! \rewards(s, \!a)$ -- seems restrictive. However, since most intrinsic rewards are strictly greater than zero (e.g., \cite{houthooft2016vime,fu2017ex2}), it can be easily satisfied in practice.

The full proof is deferred to the longer version of this paper. Here, we only focus on the insights conveyed by Theorem~\ref{theorem:critic_theorem}. According to the definition of the Bellman optimality operator (Section~\ref{Reinforcement Learning in Continuous Action Spaces}), $\expectation_{\rho} [ \mathcal{T}^{\mathrm{max}}_{\rewards} Q^{\pi}_{\rewards} - Q^{\pi}_{\rewards} ]$ quantifies the improvement on $Q^{\pi}_{\rewards}$ after performing one value iteration step (w.r.t. $\rewards$) \cite{bellman1966dynamic}. Depending on the state-action distribution $\rho$ used to compute expectation, this quantity becomes a proxy of different measures. Specifically, $\expectation_{\rho_{\pi}} [ \mathcal{T}^{\mathrm{max}}_{\rewards} Q^{\pi}_{\rewards} - Q^{\pi}_{\rewards} ]$ (where the expectation is calculated w.r.t. $\rho_{\pi}$) represents the expected improvement of the target policy, which is our ultimate learning goal and hence is a proxy of \emph{training stability} given learning is stable if non-decreasing.  $\expectation_{\rho_{\mu}} [ \mathcal{T}^{\mathrm{max}}_{\rewards} Q^{\pi}_{\rewards} - Q^{\pi}_{\rewards} ]$ (where the expectation is with regard to $\rho_{\mu}$) measures the training \emph{effectiveness} in the sense that the better $\mu$ performs (i.e. $\expectation_{\rho_{\mu}} [ Q^{\pi}_{\rewards} ]$), the higher the quality of collected samples, given $\mu$ is used to interact with the environment.

Since both proxies measure the near-future improvement of $Q^{\pi}_{\rewards}$, higher values are preferable. 
In particular, if of both the lower bounds are guaranteed to be positive, ADAC's performance will be monotonously increasing during training even with intrinsic reward. 
Note that the lower bound of both proxies are inter-defined on each other through their first terms, we are not able to comment on the exact value of either alone. Still, we can share several encouraging observations, which are related to the remaining parts of both bounds.
First, the remaining part of Eq~\eqref{eq: first theoretical result} (i.e. $\expectation_{\rho_{\pi}} [\rewards] - \expectation_{\rho_{\mu}} [\rewards]$) is always non-negative since $\pi$ is optimized to maximize the cumulative reward of $\rewards$ while $\mu$ is not. 
This suggests that even if $\expectation_{\rho_{\mu}} [ \mathcal{T}^{\mathrm{max}}_{\rewards} Q^{\pi}_{\rewards} - Q^{\pi}_{\rewards} ]$ is negative, Eq~\eqref{eq: first theoretical result} might still be positive, which guarantees training stability.
Next, we demonstrate that the remaining term of Eq~\eqref{eq: second theoretical result} (i.e. $\expectation_{\rho_{\pi}} [ \rewards - \rewards' ]$) is very likely to be high during ADAC's training process. 
Promised by the policy co-training approach (Section~\ref{Policy Co-training}), when following $\mu$ we would frequently visit high-probability state-action pairs in $\rho_{\pi}$.
Since most intrinsic reward functions are designed to assign small rewards to states that are frequently visited, $\mathbb{E}_{\rho_{\pi}} [\rewards' - \rewards] = \mathbb{E}_{\rho_{\pi}} [\rewards^{in}]$ (the negation of the remaining term of Eq~\eqref{eq: second theoretical result}) would be small. 
Finally, bringing the two bounds together, we can show that when $\pi$ and $\mu$ are jointly optimized with the presence of intrinsic rewards, the performance (w.r.t. the environment-defined reward $\rewards$) of both policies is very unlikely to drop catastrophically, thanks to the remaining terms of both bounds, which are either positive or very likely to be high. 

\section{Experiments}
\label{experiments}

In this section, we take gradual steps to analyze and illustrate our proposed method ADAC. Specifically, We first investigate the behavior of our analogous disentangled behavior policy $\mu$ (Section~\ref{Analysis of the stochastic behavior policy}). Next, we perform an empirical evaluation of ADAC without intrinsic rewards on 14 standard continuous-control benchmarks (Section~\ref{Benchmark tests}). Finally, encouraged by its promising performance and to further justify the critic bounding method, we examine ADAC with intrinsic rewards in 4 sparse-reward and hence exploration-heavy environments  (Section~\ref{Evaluation in Sparse-Reward Environments}). Throughout this paper, we highlight two benefits from the analogous disentangled nature of ADAC: (i) avoiding unnecessary trade-offs between current optimality and exploration (i.e. a more expressive and effective behavior policy); (ii) natural compatibility with intrinsic rewards without altering environment-defined optimality. In this context, the first two subsections are devoted to demonstrating the first benefit and the last subsection is dedicated for the second.

\subsection{Analysis of Analogous Disentangled Behavior Policy}
\label{Analysis of the stochastic behavior policy}

Since we are largely motivated by the potential luxury of designing an expressive exploration strategy offered by the disentangled nature of our framework, it is natural we are first interested in investigating how well our behavior policy lives up to this expectation.
Yet as discussed in Section 4.2, in order to aid stable policy updates, we specifically put some restrains on our behavior policy, deliberately making it analogous of the target policy, which means our behavior policy may not be as expressive as otherwise. Given this, we start with investigating whether our behavior policy is still expressive enough, which is measured by its \emph{coverage} (i.e. does it explore a wide enough action/policy space outside the current target policy). To further examine the influence of our added constraints, we study the policy network's \emph{stability} (i.e. does the \emph{policy co-training} lowers the bias between two policies and stabilize $\pi$'s learning process). Finally, we focus on the effectiveness of our behavior policy by measuring the overall \emph{performance} of ADAC  (i.e. does ADAC's exploration strategy efficiently lead the target policy to iteratively converge to a more desirable local optimum).

\begin{figure}[t]
    \centering
    \includegraphics[width=\columnwidth]{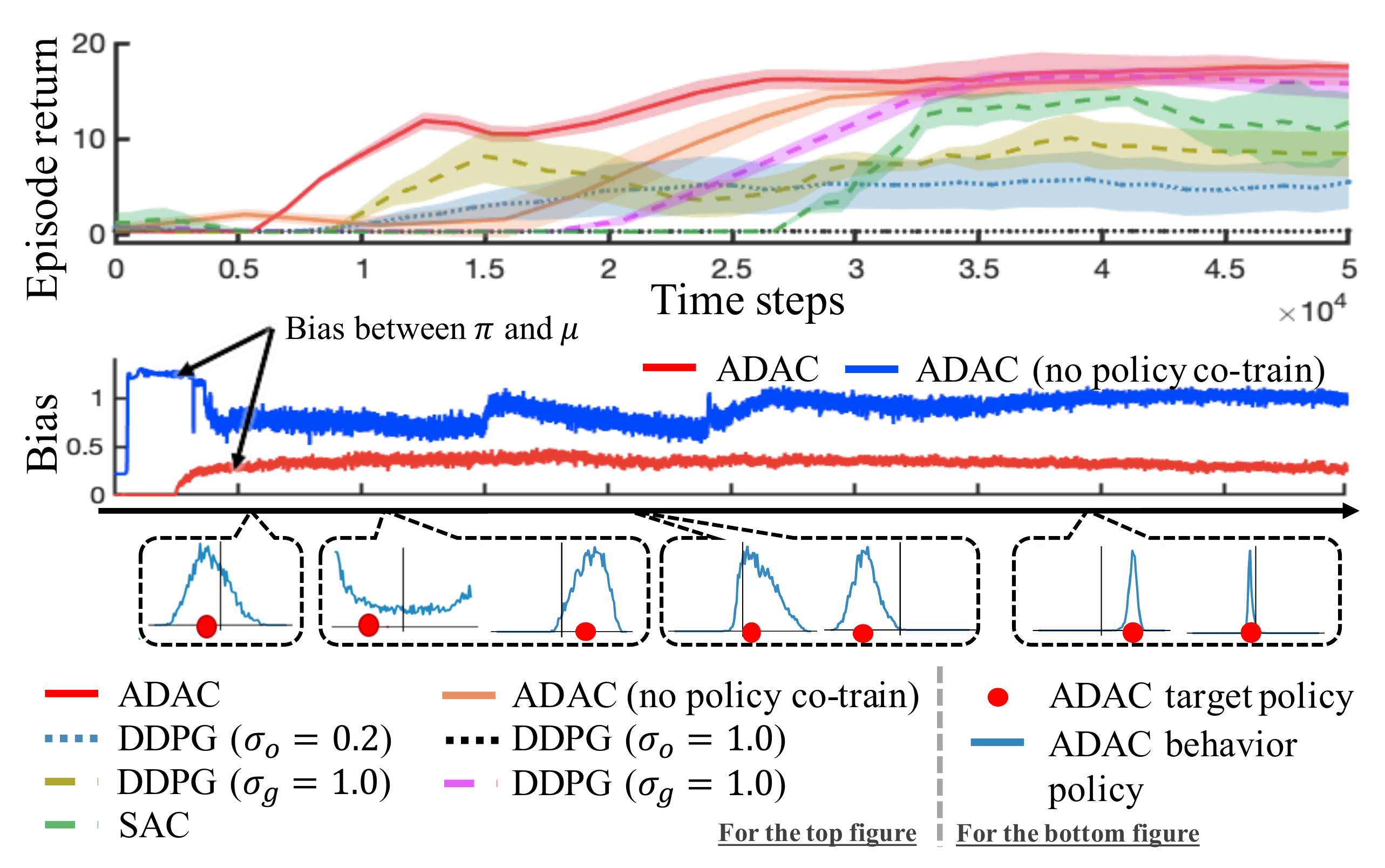}
    \caption{Learning curves of ADAC (with base model DDPG) and baselines on the modified CartPole environment. In addition, ADAC's target (red dots) and behavior policies (solid blue curves) at different timesteps are plotted below the learning curves.}
    \label{fig:behavior_policy_illustration}
\end{figure} 

\textbf{Setup} $\;$
For the sake of easy illustration, we choose a straightforward environment, namely CartPole \cite{openaigym}, as our demonstration bed. The goal in this environment is to balance the pole that is attached to a cart by applying left/right force. For the compatibility with continuous control and a better modeling of real-world implications, we modify CartPole's original discrete action space and add effort penalty to the rewards as specified in Table~\ref{table_cartpole}.
To demonstrate the advantages of our behavior policy, we choose DDPG with two commonly-used existing exploration strategies as the main baselines, i.e., Gaussian noise ($\sigma_{g}$) and Ornstein-Uhlenbeck process noise ($\sigma_{o}$) \cite{UOProcess}, both with two variance levels $0.2$ and $1.0$. For the virtue of fair comparison, we only present DDPG-based ADAC here (or simply ADAC later in this subsection). To further demonstrate the benefits from disentanglement, we choose SAC as another baseline. As discussed earlier in related works, SAC similarly also utilizes energy-based policies, yet opposite to our approach its exploration is embedded into its target policy.

\textbf{Empirical Insights} $\;$ To minimize distraction, our discussion starts with closely examining ADAC's behavior and target policy alone. First, see the cells at the bottom of Figure~\ref{fig:behavior_policy_illustration}, which are snapshots of the behavior and target policy at different training stages. As suggested by the wide bell shape of the solid blue curves ($\mu$) at the first cell, our behavior policy acts curiously when ignorant about the environment, extensively exploring all possible actions including those that are far away from the target policy (represented by the red dots). Yet having such a broad coverage alone is still not sufficient to overcome the beginning trap of getting stuck in the deceiving local-optimum of constantly exerting $a \in [-0.5, 0.5]$. As suggested by the bimodal shape of the solid blue curve ($\mu$) in the second cell, after acquiring a preliminary understanding of the environment the agent starts to form preference for some actions when exploring. Almost at the same time, the target policy no longer stays close to 0.0 (represented by the intersection of the two axes), suggesting that the behavior policy is effective in leading the target policy towards a more desirable place. This can be further corroborated by what is suggested from the third and fourth cell. In the late stage, besides being able to balance the pole, our agent even manages to learn exerting actions with small absolute value from time to time to avoid the effort penalty.

Other than its expressiveness, stability critically influences ADAC's overall performance, which is by design controlled by the proposed \emph{policy co-training} approach. To examine its effect, we perform an ablation study about it. To be more specific, we compare ADAC with ADAC without \emph{policy co-training}.\footnotemark[6] \footnotetext[6]{
When policy co-training is not used, we make two copies of ADAC's original policy network (which encodes both $\pi$ and $\mu$) to represent $\pi$ and $\mu$, respectively. Gradients of the policies are applied only to their corresponding networks.
} The effect of critic bounding is measured by the \emph{bias} between $\pi$ and $\mu$, which is shown in the middle of Figure~\ref{fig:behavior_policy_illustration}. We can see that ADAC has much lower bias than its variant without policy co-training. Additionally, policy co-training does not affect the expressiveness of $\mu$, which is suggested by the behavior policies rendered below in Figure~\ref{fig:behavior_policy_illustration}. 

Finally we move our attention to the learning curves in Figure~\ref{fig:behavior_policy_illustration}: ADAC exceeds baselines in both learning efficiency (i.e. being the first to consistently accumulate positive rewards) and final performance. Unlike our  behavior policy, exploration through random noise is unguided, resulting in either wasted exploration on unpromising regions or insufficient exploration on rewarding areas. This largely explains the noticeable performance gap between DDPG with random noise and ADAC. On the other side, SAC bears an expressive policy similar to our behavior policy. However, suffering from no separate behavior policy, to aid exploration, SAC has to consistently take sub-optimal actions into account, adversely affecting its policy improvement process. In other words, different from ADAC, SAC cannot fully exploits its learned knowledge of the environment (i.e. its value functions) to construct its target policy, leading to a performance inferior to ADAC's.

\begin{table*}[t]
\caption{Specifications of our action and reward designs for the modified CartPole task. The original task consists of two discrete actions \emph{left} and \emph{right}, each pushing the cart towards its corresponding direction. We converted them into a single-dimension continuous action.}
\centering
{\fontsize{9}{9}\selectfont
\begin{tabular}{cc}
    \toprule
    Action ($a \in [-1,1]$) & Reward ($r \in \mathbb{R}$) \\
    \midrule\midrule
    $a = \left \{ \begin{array}{cc}{\text { left }} & {a<-0.5} \\ {p(\text {left})=p(\text {right})=0.5} & {a \in [0.5,0.5]} \\ {\text { right }} & {a>0.5}\end{array}\right.$ & $r = - 0.1|a| - 0.05a^{2} + \left\{\begin{array}{cc}{-1.0} & {\text { episode ended }} \\ {0.1} & {\text { otherwise }}\end{array}\right.$ \\
    \bottomrule
    
\end{tabular}
}
\label{table_cartpole}
\end{table*}

\begin{table*}
\caption{Continuous-control performance in 14 benchmark environments. Average episode return ($\pm$ standard deviation) over 20 trials are reported. Bold indicates the best average episode return. $\dagger$ indicates the better performance between ADAC(TD3) and its base model TD3. Similarly, $*$ indicates the better performance between ADAC (DDPG) and its base model DDPG. In all three cases, values that are statistically insignificantly different (>0.05 in t-test) from the respective should-be indicated ones are denoted as well.}
\label{table:benchmark}
\centering
{\fontsize{9}{9}\selectfont

\begin{tabular}{ccccccc}
\toprule
Environment&ADAC (TD3)& ADAC (DDPG)& TD3& DDPG& SAC& PPO\\
\midrule\midrule
RoboschoolAnt& 2219$\pm$373& 838.1*$\pm$97.1& \textbf{2903}$\dagger\pm$666& 450.0$\pm$27.9& \textbf{2726}$\pm$652& 1280$\pm$71\\

RoboschoolHopper& \textbf{2299}$\dagger\pm$333 & 766.5*$\pm$10 & \textbf{2302}$\dagger\pm$537 & 543.8$\pm$307 & 2089$\pm$657 & 1229$\pm$345\\

RoboschoolHalfCheetah& 1578$\dagger\pm$166 & \textbf{1711}*$\pm$95& 607.2$\pm$246.2 & 441.6$\pm$120.4 & 807.0$\pm$252.6 & 1225$\pm$184.2\\

RoboschoolAtlasForwardWalk& \textbf{234.6}$\dagger \pm$55.7 & 186.7*$\pm$37.9 & 190.6$\pm$50.1& 52.63$\pm$26.2& 126.0$\pm$47.1& 107.6$\pm$29.4\\

RoboschoolWalker2d& \textbf{1769}$\dagger \pm$452 & \textbf{1564}*$\pm$651 & 995.1$\pm$146.3 & 208.7$\pm$137.1 & 1021$\pm$263 & 578.9$\pm$231.3\\

Ant& 3353$\pm$847 & 1226*$\pm$18 & 4034$\dagger \pm$517 & 370.5$\pm$223 & \textbf{4291}$\pm$1498 & 1401$\pm$168\\

Hopper& \textbf{3598}$\dagger\pm$ 374 & 374.5*$\pm$36.5& 2845$\pm$609& 38.93$\pm$0.88& \textbf{3307}$\pm$825& 1555$\pm$458\\

HalfCheetah& 9392$\pm$199& 2238*$\pm$40& 10526$\dagger\pm$2367& 1009$\pm$49& \textbf{11541}$\pm$2989& 881.7$\pm$10.1\\

Walker2d& \textbf{5122$\dagger$}$\pm$1314& 1291*$\pm$42& 4630$\dagger\pm$778& 186.2$\pm$33.3& 4067$\pm$1211& 1146$\pm$368\\

InvertedPendulum& \textbf{1000$\dagger$}$\pm$0& \textbf{1000*}$\pm$0& \textbf{1000$\dagger$}$\pm$0& \textbf{1000*}$\pm$0& \textbf{1000}$\pm$0& 98.90$\pm$2.08\\

InvertedDoublePendulum& \textbf{9359$\dagger$}$\pm$0.17& 9334*$\pm$1.39& 7665$\pm$566& 27.20$\pm$2.61& 9353$\pm$2896& 98.90$\pm$5.88\\

BipedalWalker& \textbf{309.8$\dagger$}$\pm$15.6& -52.77*$\pm$1.94& 288.4$\dagger\pm$51.25& -123.90$\pm$11.17& \textbf{307.2}$\pm$57.92& 266.9$\pm$28.52\\

BipedalWalkerHardcore& \textbf{-10.76$\dagger$}$\pm$27.70& -98.52$\pm$3.21& -57.97$\pm$21.08& -50.05*$\pm$10.27& -127.4$\pm$45.2& -105.3$\pm$22.2\\

LunarLanderContinuous& \textbf{290.0$\dagger$}$\pm$50.9& 85.67*$\pm$23.42& \textbf{289.7$\dagger$}$\pm$54.1& -65.89$\pm$96.48& 283.3$\pm$69.29& 59.32$\pm$68.44\\

\bottomrule
\end{tabular}
}
\end{table*}

\subsection{Comparison with the State of the Art}
\label{Benchmark tests}

Though well-suited for illustration, CartPole alone is not challenging and generalized enough to fully manifest ADAC's competitiveness. In this subsection, we present that ADAC can achieve state-of-the-art performance in standard benchmarks. 

\textbf{Setup} $\;$ To demonstrate the generality of our method, we construct a 14-task testbed suite composed of qualitatively diverse continuous-control environments from the OpenAI Gym toolkit \cite{openaigym}.
On top of the two baselines adopted earlier (i.e. DDPG and SAC), we further include TD3 \cite{fujimoto2018addressing}, which improves upon DDPG by addressing some of its function approximation errors, PPO \cite{schulman2017proximal}, which is regarded as one of the most stable and efficient on-policy policy gradient algorithm, and GEP-PG \cite{colas2018gep}, which combines Goal Exploration Process \cite{pere2018unsupervised} with policy gradient to perform curious exploration as well as stable learning. Though not exhaustive, this baseline suite still embodies many of the latest advancements and can be indeed deemed as the existing state-of-the-art. However, we compare with GEP-PG only in tasks adopted in their original experiments. Specifically, since the GEP part of the algorithm needs hand-crafted exploration goals, we are not able to run their model on new experiments since it is nontrivial to generalize their experiments in other tasks. 
To best reproduce the rest's performance, we use their original open-source implementations if released; otherwise, we build our own versions after the most-starred third-party implementations in GitHub.
Furthermore, to prevent over-claiming the state-of-the-art, we fine-tune their hyper-parameters around the values reported in the respective literature, but only coarsely tune the hyper-parameters introduced by ADAC.\footnotemark[7]
All experiments are run for 1 million time-steps, or until reaching performance convergence, whichever happens earlier.

\textbf{Empirical Insights} $\;$ Table~\ref{table:benchmark} corroborates that ADAC's competitiveness over existing methods stem from its disentangled nature. More importantly, these results reveal two desirable properties of ADAC's full compatibility with existing off-policy methods. First, ADAC consistently outperforms the method it is based on. As indicated by the $*$ symbols, compared to its base model, DDPG-based ADAC achieves statistically better or comparable performance on more than $93\% (13/14)$  of the benchmarks and obtains identical performance on one of the remaining two. Though not as remarkable as DDPG-based ADAC, TD-based ADAC also manages to achieve statically better or comparable performance over its base model on more than $78\% (11/14)$ of the tasks; see the $\dagger$ symbols.  Second, ADAC retains the benefits of improvements developed by the base model themselves. This is best illustrated by TD3-based ADAC's performance superiority over DDPG-based ADAC. 

\footnotetext[7]{For a fair comparison, we do not use intrinsic reward throughout this section, since most baseline approaches are not able to naturally incorporate them during learning. See Appendix~\ref{Continuous Control Benchmarks} for additional environment details, and Appendix~\ref{Hyper-parameters and Network Structure} for hyperparameters of ADAC and baselines. Full banchmark results are given in Appendix~\ref{Full Benchmark Result}.}

We would like to specially call readers' attention to our comparison of ADAC with SAC since they both use energy-based behavior policy. This comparison also reveals the benefit brought by the disentangled structure and the analogous actors and critics. ADAC (TD3) achieves better average performance over SAC on 71\%(10/14) of the benchmarks, indicating the effectiveness of our proposed analogous disentangled structure. 

Despite that GDP-PG \cite{colas2018gep} also uses the disentanglement idea, we do not compare ADAC with it across the whole 14-benchmark test suite and hence GEP-PG is not included in Table 2. This decision is made largely due to the fact that the Goal Exploration Process (GEP) in GEP-PG requires manually defining a goal space to explore, which is task dependent and can critically influence the algorithm performance. Given this, we can only compare with it on the two experiments that GEP-PG has run, of which only one overlaps with our task suit, namely HalfCheetah. In HalfCheetah, GEP-PG achieves 6118 cumulative reward, while ADAC (TD3) achieves 9392, showing superiority over GEP-PG. Furthermore as also acknowledged in its paper, GEP-PG lags behind SAC in performance, which suggests that naively disentangling the behavior policy from the target policy does not guarantee competitive performance. Rather, to design effective disentangled actor-critic, we should also pay attention to how to best restrict some components.

When considering all reported methods together, TD3-based ADAC obtains the most number of the state-of-the-art results; as indicated in bold, it is the best performer (or statistically comparable with the best) on more than $71\% (10/14)$ of the benchmarks. 

\subsection{Evaluation in Sparse-Reward Environments}
\label{Evaluation in Sparse-Reward Environments}

\begin{figure}[t]
    \centering
    \includegraphics[width=\columnwidth]{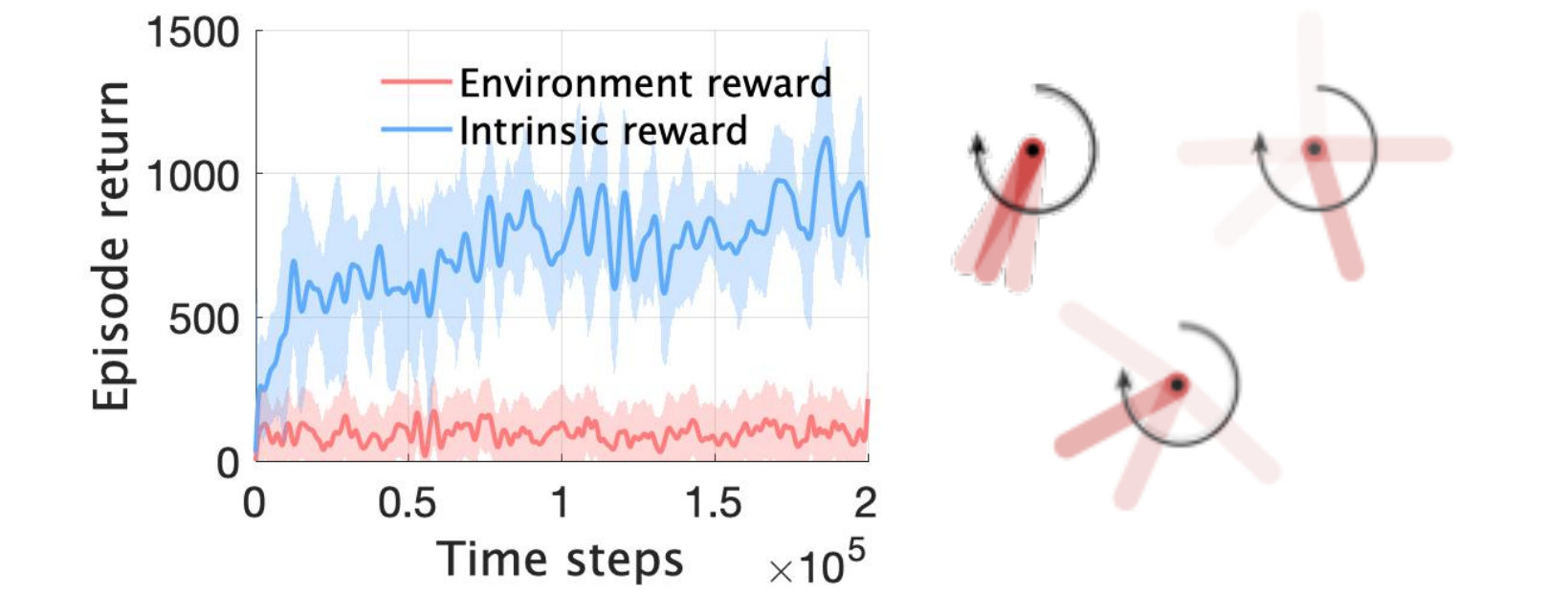}
    \caption{Illustration of how intrinsic reward contaminates the environment-defined optimality in PendulumSparse. Fooled into collecting more intrinsic rewards rather than environment rewards (see the learning curves on the left), the agent constantly alternates between spinning the pendulum and barely moving it (see the snapshots of the target policy on the right), making no real progress.}
    \label{fig:bad_pendulum}
\end{figure}

\begin{figure}[t]
    \centering
    \includegraphics[width=\columnwidth]{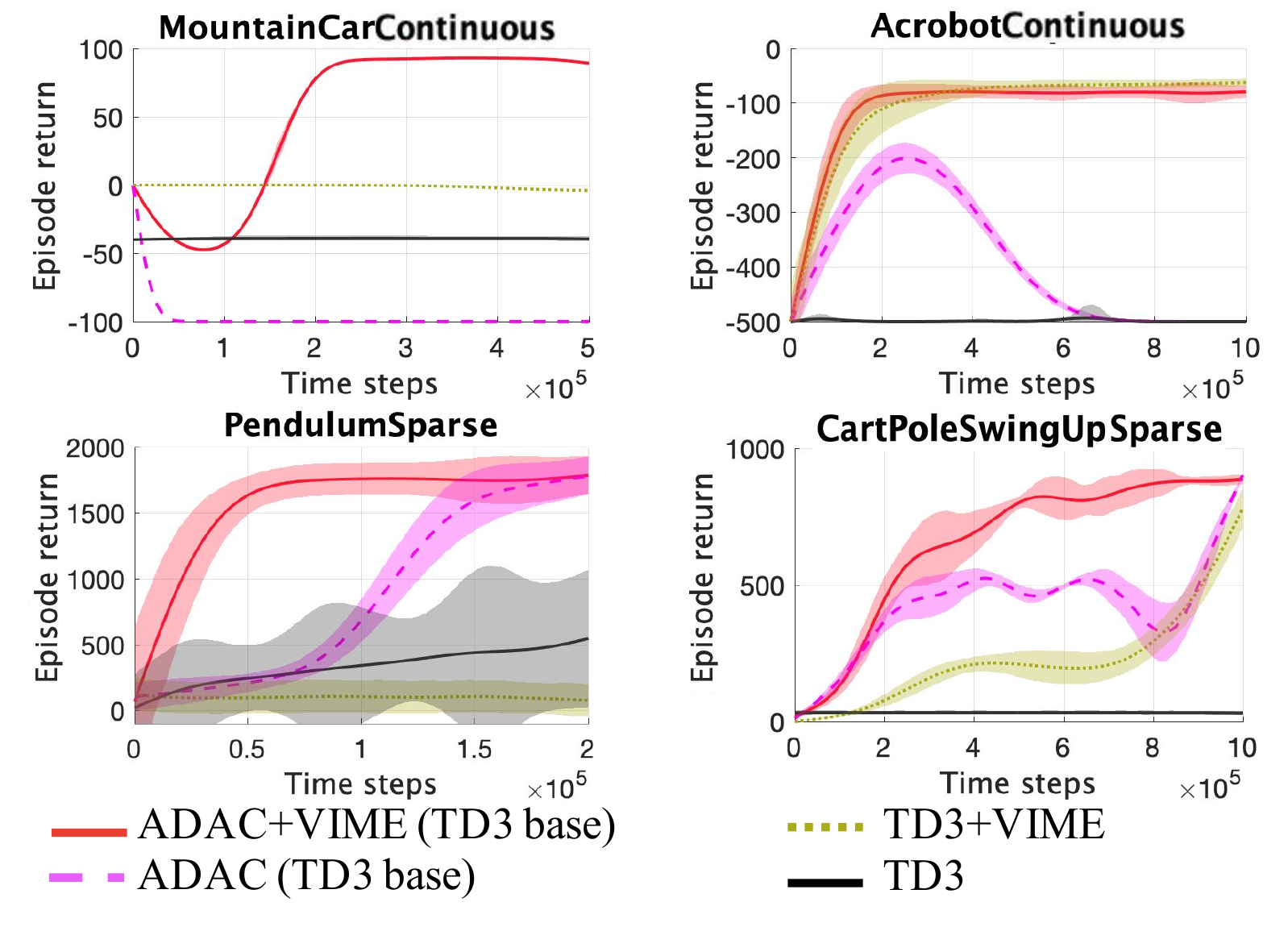}
    \caption{Learning curves for four sparse-reward tasks. Lines denote the average over 20 trials and the shaded areas represent the range of one standard deviation.}
    \label{fig:four_sparse_tasks}
\end{figure}

Encouraged by the promising results observed on the benchmarks, in this subsection we evaluate ADAC under more challenging environments, in which rewards are barely provided. This set of experiments aim to test ADAC's exploration capacity under extreme settings. Furthermore, we also see them fit as demonstration beds to present ADAC's natural compatibility with intrinsic methods. In this regard, we are particularly interested in investigating whether the disentangled nature of ADAC helps mitigate intrinsic rewards' undesirable bias effect on the environment-defined optimality.

\textbf{Setup} $\;$ To the surprise of many, sparse-reward environments turn out to be relatively unpopular in commonly-used RL toolkits. Besides including the classic MountainCarContinuous and Acrobot (after converting its action space to be countinuous), to construct a decently sized testing suite, we further hand-craft new tasks, namely PendulumSparse and CartPoleSwingUpSparse by sparsifying the rewards in the existing environments. It is achieved mainly through suppressing the original rewards until reaching some predefined threshold.\footnotemark[8] 
Due to their dependency on environment-provided rewards as feedback signals, most model-free RL algorithms suffer significant performance degradation in these sparse-reward tasks. In this situation, resorting to intrinsic methods (IM) for additional signals has been widely considered as the go-to solution. Among a wide variety of IM methods, we adopt Variational Information Maximization Exploration (VIME) \cite{houthooft2016vime} as our internal reward generator for its consistent good performance on a wide variety of exploration-challenging tasks. Considering TD3-based ADAC's superiority over DDPG-based ADAC, we only combine VIME into TD3 and TD3-based ADAC. Note when paired with ADAC, intrinsic rewards are only visible to the behavior policy. 

\footnotetext[8]{More details about the sparse-reward environments can be founded in Appendix~\ref{sparse-reward environments}.}

\textbf{Empirical Insights} $\;$ Among the four environments, PendulumSparse has the most vulnerable environment-defined optimality. The goal here is to swing the inverted pendulum up so it stays upright. As suggested by Figure~\ref{fig:bad_pendulum}, not knowing how to distinguish between intrinsic and environment rewards, VIME-augmented TD3 is completely fooled into chasing after the intrinsic rewards. In other words, the VIME-augmented TD3's understanding of what is optimal is completely off from the true environment-defined optimality. Note as demonstrated in left-bottom part of Figure~\ref{fig:four_sparse_tasks}, VIME-augmented TD3's performance even trails behind TD3's, which is an indisputable evidence that the bias introduced by IM can be detrimental and should be addressed whenever possible.
In contrast, thanks to its disentangled nature, VIME-augmented ADAC only perceives intrinsic rewards in its behavior policy, which means its target policy always remains optimal with regards to our current knowledge about environment rewards. Because of this, VIME-augmented manages to consistently solve this exploration-challenging task. ADAC's natural compatibility with VIME is further corroborated by the results in the remaining 3 tasks. As suggested by the the complete Figure~\ref{fig:four_sparse_tasks}, VIME-augmented ADAC consistently surpasses all reported alternatives by a large margin in terms of both convergence speed and final performance.

\section{Conclusion}
\label{conclusion}

We present
Analogous Disentangled Actor-Critic (ADAC), an off-policy reinforcement learning framework that explicitly disentangles the behavior and target policy. 
Compared to prior work, to stabilize model updates, we restrain our behavior policy and its corresponding critic to be analogous of their target counterparts. Thanks to its disentangled and analogous nature, environment-reward-only ADAC achieves the state-of-the-art results in 10 out of 14 continuous control benchmarks. Moreover, ADAC is naturally compatible with intrinsic rewards, outperforming alternatives in exploration-challenging tasks.

\vspace{1em}
\noindent \textbf{Acknowledgements} $\;$ This work is partially supported by NSF grants \#IIS-1943641, \#IIS-1633857, \#CCF-1837129, DARPA XAI grant \#N66001-17-2-4032, UCLA Samueli Fellowship, and gifts from Intel and Facebook Research.


\clearpage

\bibliographystyle{ACM-Reference-Format}  
\bibliography{references}  


\clearpage

\onecolumn

\appendix

\section*{}
\begin{center}
    {\huge \bf Supplementary Material}
\end{center}
\vspace{4em}

\newcommand\numberthis{\addtocounter{equation}{1}\tag{\theequation}}

\section{Theoretical Results}
\label{Theoretical Results}

This section provides the full proof of Theorem~\ref{theorem:critic_theorem} that is the guarantee of the training stability as well as that of the training effectiveness of the \emph{critic bounding} approach (Section~\ref{Critic Bounding}).

\vspace{0.5em}
\noindent \emph{Proof of Theorem~\ref{theorem:critic_theorem}}
\vspace{0.2em}

We define $Q^{\pi}_{*}$ as the optimal value function with respect to policy $\pi$ and reward $\rewards$, i.e., $Q^{\pi}_{*} = \mathcal{T}^{\pi}_{\rewards} Q^{\pi}_{*}$. We further define $Q^{\mu}_{\rewards'}$ as the optimal value function with respect to $\mu$ and $\rewards'$ (i.e., $Q^{\mu}_{\rewards'} = \mathcal{T}^{\mu}_{\rewards'} Q^{\mu}_{\rewards'}$). Our proof is built upon the foundation result stated in the following lemma. For the sake of a smoother presentation, we defer its proof after we finish proving the theorem. 

\vspace{0.5em}
\begin{lemma}
\label{lemma: foundations used in proving the theorem}
Under the definitions and assumptions made in Theorem~\ref{theorem:critic_theorem} and the above paragraph, we have the following result
    \begin{align*}
        Q^{\mu}_{\rewards'} - Q^{\pi}_{*} & = \left [ (\mathcal{I} - \gamma \mathcal{P}^{\mu})^{-1} - (\mathcal{I} - \gamma \mathcal{P}^{\pi})^{-1} \right ] (\mathcal{T}^{max}_{\rewards} Q^{\pi}_{\rewards} - Q^{\pi}_{\rewards}) - (\mathcal{I} - \gamma \mathcal{P}^{\mu})^{-1} (\mathcal{T}^{max}_{\rewards} Q^{\pi}_{\rewards} - \mathcal{T}^{\mu}_{\rewards'} Q^{\pi}_{\rewards}). \numberthis \label{eq:proof_note4}
    \end{align*}
\end{lemma}
\vspace{0.5em}

Recall that $Q^{\mu}_{\rewards'}$ and $Q^{\pi}_{*}$ are the optimal value function with respect to $\langle \mu, \rewards' \rangle$ and $\langle \pi, \rewards \rangle$, respectively. By definition, we have $Q^{\mu}_{\rewards'} = (\mathcal{I} - \gamma \mathcal{P}^{\pi})^{-1} \rewards'$ and $Q^{\pi}_{*} = (\mathcal{I} - \gamma \mathcal{P}^{\pi})^{-1} \rewards$ (since $(\mathcal{I} - \gamma \mathcal{P}^{\pi})^{-1} \rewards = \sum_{t = 0}^{\infty} \gamma^{t} \mathcal{P}^{t} \rewards = Q^{\pi}_{*}$).

\textbf{Result on training effectiveness} (i.e. Eq~\eqref{eq: second theoretical result}) $\;$ We are now ready to prove the second result stated in the theorem. Since $\mathcal{T}^{\mu}_{\rewards'} Q^{\pi}_{\rewards} \geq \mathcal{T}^{max}_{\rewards} Q^{\pi}_{\rewards}$, we have $(\mathcal{T}^{max}_{\rewards} Q^{\pi}_{\rewards} - \mathcal{T}^{\mu}_{\rewards'} Q^{\pi}_{\rewards}) \leq 0$. Plug in Eq~\eqref{eq:proof_note4} and use the equality 
    \begin{align*}
        Q^{\mu}_{\rewards'} - Q^{\pi}_{*} = (\mathcal{I} - \gamma \mathcal{P}^{\pi})^{-1} (\rewards' - \rewards),
    \end{align*}
\noindent we have
    \begin{align*}
        (\mathcal{I} - \gamma \mathcal{P}^{\mu})^{-1} & (\mathcal{T}^{max}_{\rewards} Q^{\pi}_{\rewards} - Q^{\pi}_{\rewards}) \geq (\mathcal{I} - \gamma \mathcal{P}^{\pi})^{-1} (\mathcal{T}^{max}_{\rewards} Q^{\pi}_{\rewards} - Q^{\pi}_{\rewards}) + (\mathcal{I} - \gamma \mathcal{P}^{\pi})^{-1} (\rewards - \rewards'),
    \end{align*}
\noindent which is equivalent to the second result stated in the theorem (Eq~\eqref{eq: second theoretical result}).

\vspace{0.5em}
\textbf{Result on training stability} (i.e. Eq~\eqref{eq: first theoretical result}) $\;$
To prove the first result stated in the theorem, we start from rearranging Eq~\eqref{eq:proof_note4}:
    \begin{align*}
        \big [ (\mathcal{I} - \gamma \mathcal{P}^{\pi})^{-1} - (\mathcal{I} - \gamma \mathcal{P}^{\mu})^{-1} \big ] (\mathcal{T}^{max}_{\rewards} Q^{\pi}_{\rewards} - Q^{\pi}_{\rewards}) & = - (\mathcal{I} - \gamma \mathcal{P}^{\mu})^{-1} ((\mathcal{I} - \gamma \mathcal{P}^{\mu}) (Q^{\mu}_{\rewards'} - Q^{\pi}_{*}) + \mathcal{T}^{max}_{\rewards} Q^{\pi}_{\rewards} - \mathcal{T}^{\mu}_{\rewards'} Q^{\pi}_{\rewards}) \\
        & = - (\mathcal{I} - \gamma \mathcal{P}^{\mu})^{-1} (\rewards' + \gamma \mathcal{P}^{\mu} Q^{\pi}_{*} - Q^{\pi}_{*} + \mathcal{T}^{max}_{\rewards} Q^{\pi}_{\rewards} - \mathcal{T}^{\mu}_{\rewards'} Q^{\pi}_{\rewards}) \\ 
        & = - (\mathcal{I} - \gamma \mathcal{P}^{\mu})^{-1} (\rewards + \gamma \mathcal{P}^{\mu} Q^{\pi}_{*} - Q^{\pi}_{*} + \gamma \mathcal{P}^{\pi} Q^{\pi}_{\rewards} - \gamma \mathcal{P}^{\mu} Q^{\pi}_{\rewards}) \\
        & \overset{(a)}{\geq} - (\mathcal{I} - \gamma \mathcal{P}^{\mu})^{-1} (\gamma \mathcal{P}^{\pi} Q^{\pi}_{\rewards} - \gamma \mathcal{P}^{\mu} Q^{\pi}_{\rewards}) \\
        & \overset{(b)}{\geq} - \gamma (\mathcal{I} - \gamma \mathcal{P}^{\mu})^{-1} (\mathcal{P}^{\pi} - \mathcal{P}^{\mu}) (\mathcal{I} - \gamma \mathcal{P}^{\pi})^{-1} \rewards \\
        & = \big [ (\mathcal{I} - \gamma \mathcal{P}^{\pi})^{-1} - (\mathcal{I} - \gamma \mathcal{P}^{\mu})^{-1} \big ] \rewards, \numberthis \label{eq:proof_note7}
    \end{align*}
\noindent where $(a)$ uses the inequality $\mathcal{P}^{\mu} Q^{\pi}_{*} \leq \mathcal{P}^{\pi} Q^{\pi}_{*}$, and $(b)$ follows from $Q^{\pi}_{\rewards} \leq Q^{\pi}_{*} = (\mathcal{I} - \gamma \mathcal{P}^{\pi})^{-1} \rewards$. Rewriting Eq~\eqref{eq:proof_note7} gives us the first result stated in the theorem (Eq~\eqref{eq: first theoretical result}):
    \begin{align*}
        (\mathcal{I} - & \gamma \mathcal{P}^{\pi})^{-1} (\mathcal{T}^{max}_{\rewards} Q^{\pi}_{\rewards} - Q^{\pi}_{\rewards}) \geq (\mathcal{I} - \gamma \mathcal{P}^{\mu})^{-1} (\mathcal{T}^{max}_{\rewards} Q^{\pi}_{\rewards} - Q^{\pi}_{\rewards}) + \big [ (\mathcal{I} - \gamma \mathcal{P}^{\pi})^{-1} - (\mathcal{I} - \gamma \mathcal{P}^{\mu})^{-1} \big ] \rewards. 
    \end{align*}
$\hfill\square$

\vspace{1.0em}
\begin{proof}[Proof of Lemma~\ref{lemma: foundations used in proving the theorem}]
Before delving into the detailed derivation, we make the following clarifications. First, although $\pi$ is a greedy policy w.r.t. $Q^{\pi}_{\rewards}$, $Q^{\pi}_{\rewards}$ is not the optimal value function w.r.t. $\pi$ and $\rewards$. In other words, $\mathcal{T}^{max}_{\rewards} Q^{\pi}_{\rewards} = \mathcal{T}^{\pi}_{\rewards} Q^{\pi}_{\rewards}$ is guaranteed to hold yet we might have $Q^{\pi}_{\rewards} \neq \mathcal{T}^{\pi}_{\rewards} Q^{\pi}_{\rewards}$. Second, in both the theorem and the proof, we omit the state-action notation (e.g., $Q(s,a)$) for the sake of simplicity.

We begin from the difference between the respective optimal value function with regard to $\mathcal{T}^{\pi}_{\rewards'}$ and $\mathcal{T}^{\pi}_{\rewards}$:

\begin{align*}
    Q^{\mu}_{\rewards'} - Q^{\pi}_{*} & \overset{(a)}{=} \mathcal{T}^{\mu}_{\rewards'} Q^{\mu}_{\rewards'} - \mathcal{T}^{\mu}_{\rewards'} Q^{\pi}_{\rewards} + \mathcal{T}^{max}_{\rewards} Q^{\pi}_{\rewards} - \mathcal{T}^{\pi}_{\rewards} Q^{\pi}_{*} - (\mathcal{T}^{\mu}_{\rewards'} Q^{\mu}_{\rewards'} - Q^{\mu}_{\rewards'} + \mathcal{T}^{max}_{\rewards} Q^{\pi}_{\rewards} - \mathcal{T}^{\mu}_{\rewards'} Q^{\pi}_{\rewards}) \\
    & \overset{(b)}{=} \gamma \mathcal{P}^{\mu} (Q^{\mu}_{\rewards'} - Q^{\pi}_{*} + Q^{\pi}_{*} - Q^{\pi}_{\rewards}) + \gamma \mathcal{P}^{\pi} (Q^{\pi}_{\rewards} - Q^{\pi}_{*}) - (\mathcal{T}^{max}_{\rewards} Q^{\pi}_{\rewards} - \mathcal{T}^{\mu}_{\rewards'} Q^{\pi}_{\rewards}), \numberthis \label{eq:proof_result_1}
\end{align*}

\noindent where $\mathcal{P}^{\pi}$ is the state probability transition operator with respect to the environment dynamics and policy $\pi$; $(a)$ uses the equality $Q^{\pi}_{*} = \mathcal{T}^{\pi}_{\rewards} Q^{\pi}_{*}$; $(b)$ adopts the fact that $\mathcal{T}^{max}_{\rewards} Q^{\pi}_{\rewards} = \mathcal{T}^{\pi}_{\rewards} Q^{\pi}_{\rewards}$. Combining the terms $Q^{\mu}_{\rewards'} - Q^{\pi}_{*}$ and $Q^{\pi}_{*} - Q^{\pi}_{\rewards}$ gives us
    \begin{align*}
        (\mathcal{I} - \gamma \mathcal{P}^{\mu}) & (Q^{\mu}_{\rewards'} - Q^{\pi}_{*}) = (\gamma \mathcal{P}^{\mu} - \gamma \mathcal{P}^{\pi}) (Q^{\pi}_{*} - Q^{\pi}_{\rewards}) - (\mathcal{T}^{max}_{\rewards} Q^{\pi}_{\rewards} - \mathcal{T}^{\mu}_{\rewards'} Q^{\pi}_{\rewards}), \numberthis \label{eq:proof_note3}
    \end{align*}
\noindent where $\mathcal{I}$ is the identity operator, i.e. $\mathcal{I} Q = Q$. We define $(\mathcal{I} - \gamma \mathcal{P})^{-1} \overset{def}{=} \mathcal{I} + \sum_{t=1}^{\infty} \gamma^{t} \mathcal{P}^t$. By definition, given the initial state-action distribution $\beta$, $(\mathcal{I} - \gamma \mathcal{P}^{\pi})^{-1} \beta$ is the state-action marginal distribution with respect to $\beta$ and policy $\pi$. We can easily verify that $(\mathcal{I} - \gamma \mathcal{P})^{-1} (\mathcal{I} - \gamma \mathcal{P}) = \mathcal{I}$ and $(\mathcal{I} - \gamma \mathcal{P}) (\mathcal{I} - \gamma \mathcal{P})^{-1} = \mathcal{I}$ since by definition, $\gamma < 1$.

Next, we derive the connection between $Q^{\pi}_{*} - Q^{\pi}_{\rewards}$ and $\mathcal{T}^{\rewards} Q^{\pi}_{\rewards} - Q^{\pi}_{\rewards}$, which is closely related to the result given by Munos et al. \cite{munos2007performance}:
    \begin{align*}
        (\mathcal{I} - \gamma \mathcal{P}^{\pi}) (Q^{\pi}_{*} - Q^{\pi}_{\rewards}) & = Q^{\pi}_{*} - Q^{\pi}_{\rewards} - \gamma \mathcal{P}^{\pi} Q^{\pi}_{*} + \gamma \mathcal{P}^{\pi} Q^{\pi}_{\rewards} \\
        & = \rewards + \gamma \mathcal{P}^{\pi} Q^{\pi}_{\rewards} - (\rewards + \gamma \mathcal{P}^{\pi} Q^{\pi}_{*}) + Q^{\pi}_{*} - Q^{\pi}_{\rewards} \\
        & = \mathcal{T}^{\pi}_{\rewards} Q^{\pi}_{\rewards} - \mathcal{T}^{\pi}_{\rewards} Q^{\pi}_{*} + Q^{\pi}_{*} - Q^{\pi}_{\rewards} \\
        & = \mathcal{T}^{max}_{\rewards} Q^{\pi}_{\rewards} - Q^{\pi}_{\rewards},
    \end{align*}
\noindent where the result $\mathcal{T}^{\pi}_{\rewards} Q^{\pi}_{*} = Q^{\pi}_{*}$ and $\mathcal{T}^{\pi}_{\rewards} Q^{\pi}_{\rewards} = \mathcal{T}^{max}_{\rewards} Q^{\pi}_{\rewards}$ are used. Plug in Eq~\eqref{eq:proof_note3}, we have
    \begin{align*}
        (\mathcal{I} - \gamma \mathcal{P}^{\mu}) (Q^{\mu}_{\rewards'} - Q^{\pi}_{*}) = (\gamma \mathcal{P}^{\mu} - \gamma \mathcal{P}^{\pi}) (\mathcal{I} - \gamma \mathcal{P}^{\pi})^{-1} (\mathcal{T}^{max}_{\rewards} Q^{\pi}_{\rewards} - Q^{\pi}_{\rewards}) - (\mathcal{T}^{max}_{\rewards} Q^{\pi}_{\rewards} - \mathcal{T}^{\mu}_{\rewards'} Q^{\pi}_{\rewards}).
    \end{align*}
Combining the above equation with Eq~\eqref{eq:proof_result_1}, we get
    \begin{align*}
        Q^{\mu}_{\rewards'} - Q^{\pi}_{*} & = (\mathcal{I} - \gamma \mathcal{P}^{\mu})^{-1}  (\gamma \mathcal{P}^{\mu} - \gamma \mathcal{P}^{\pi}) (\mathcal{I} - \gamma \mathcal{P}^{\pi})^{-1} (\mathcal{T}^{max}_{\rewards} Q^{\pi}_{\rewards} - Q^{\pi}_{\rewards}) - (\mathcal{I} - \gamma \mathcal{P}^{\mu})^{-1} (\mathcal{T}^{max}_{\rewards} Q^{\pi}_{\rewards} - \mathcal{T}^{\mu}_{\rewards'} Q^{\pi}_{\rewards}) \\
        & = \left [ (\mathcal{I} - \gamma \mathcal{P}^{\mu})^{-1} - (\mathcal{I} - \gamma \mathcal{P}^{\pi})^{-1} \right ] (\mathcal{T}^{max}_{\rewards} Q^{\pi}_{\rewards} - Q^{\pi}_{\rewards}) - (\mathcal{I} - \gamma \mathcal{P}^{\mu})^{-1} (\mathcal{T}^{max}_{\rewards} Q^{\pi}_{\rewards} - \mathcal{T}^{\mu}_{\rewards'} Q^{\pi}_{\rewards}).
    \end{align*}
\end{proof}

\section{Algorithmic Details of ADAC}
\label{Algorithmic Details of ADAC}

\begin{algorithm}[t]
\caption{Adventurous Actor-Critic (ADAC) with DDPG as the base model}
\label{alg:ADAC_DDPG}
{\fontsize{9}{9} \selectfont
\begin{algorithmic}[1]

\STATE{\textbf{input:} environment $\mathcal{E}$, batch size $M$, maximum episode length $T_{\mathrm{max}}$, dimension of the action space $d$, mini-batch size $K$, $n_{\xi}$, and $\tau$.} 

\STATE{\textbf{initialize:} networks $Q^{\mathrm{tar}}_{\phi}$, $Q^{\mathrm{beh}}_{\psi}$, and $f_{\varphi}$; target networks $Q^{\mathrm{tar}}_{\phi^{\prime}}$, $Q^{\mathrm{beh}}_{\psi^{\prime}}$, and $f_{\varphi^{\prime}}$; replay buffer $\mathbb{B}$. $Q^{\mathrm{tar}}_{\phi}$ and $Q^{\mathrm{beh}}_{\phi}$ correspond to $Q^{\pi}_{\rewards}$ and $Q^{\mu}_{\rewards'}$ in the main text, respectively. $\phi^{\prime} := \phi$; $\psi^{\prime} := \psi$; $\varphi^{\prime} := \varphi$.}

\STATE{Define the deterministic target policy $\pi$ and the stochastic behavior policy $\mu$ with $f_{\varphi}$:
    \vspace{-0.3em}
    \begin{gather}
        \pi (s) := f_{\varphi} (s, \xi) \mid_{\xi = [0,0,\dots,0]^{T}}, \label{eq: ADAC (DDPG) pi definition} \\
        \mu (\cdot \mid s) := f_{\varphi} (s, \xi) \mid_{\xi \sim \mathcal{N} (\mathbf{0}, \mathbf{I})} + \mathcal{N} (\mathbf{0}, d/K), \label{eq: ADAC (DDPG) mu definition}
    \end{gather}
    \vspace{-1em}
}
\STATE{where $s$ and $\xi$ are input to the neural network $f_{\varphi}$ (see Figure~\ref{fig: policy network architecture} for its structure). For the target policy, $\xi$ is fixed as a $n_{\xi}$-dimensional zero vector. To sample an action from the behavior policy, we first sample $\xi$ from $\mathcal{N} (\mathbf{0}, \mathbf{I})$, feed it into $f_{\varphi}$ together with $s$, and add Gaussian noise $\mathcal{N} (\mathbf{0}, d / K)$.}

\vspace{0.3em}

\REPEAT

\STATE{Reset the environment $\mathcal{E}$ and receive the initial state $s_{0}$.}

\FOR{$t = 0, \ldots, T_{\mathrm{max}} - 1$}

\STATE{Sample $a_{t}$ from $\mu (\cdot \mid s_{t})$ according to Eq~\eqref{eq: ADAC (DDPG) mu definition}.}

\STATE{Execute $a_{t}$ in $\mathcal{E}$ and observe environment reward $r_{t}$, intrinsic reward $r_{t}^{\mathrm{in}}$ (from any intrinsic motivation approach, or simply set to zero), and the next state $s_{t+1}$.}

\STATE{Store tuple $\left(s_{t}, a_{t}, r_{t}, r_{t}^{in}, s_{t+1}\right)$ in replay buffer $\mathbb{B}$.}

\STATE{Call \textbf{training procedure}}

\STATE{\textbf{Break} if the current episode terminates on $s_{t+1}$.}

\ENDFOR

\UNTIL{$\varphi^{\prime}$ converge or reaching the pre-defined number of training steps}

\STATE{\textbf{return} $\varphi^{\prime}$}

\STATE{}

\STATE{\textbf{training procedure}}

\STATE{\hspace{\algorithmicindent} Sample a minibatch of $M$ samples $\{ ( s_{i}, a_{i}, r_{i}, r_{i}^{in}, s'_{i} ) \}_{i = 1}^{M}$ from $\mathbb{B}$.}

\STATE{\hspace{\algorithmicindent} Update $\phi$ and $\psi$ by minimizing the losses
    \begin{gather*}
        \mathcal{L} (\phi) = \frac{1}{M} \sum_{i = 1}^{M} (Q^{\mathrm{tar}}_{\phi}(\rlstate_{i}, \action_{i}) - y_{i})^{2}, \\
        \mathcal{L}' (\psi) = \frac{1}{M} \sum_{i = 1}^{M} (Q^{\mathrm{beh}}_{\psi}(\rlstate_{i}, \action_{i}) - y_{i}')^{2},
    \end{gather*}
}
\STATE{\hspace{\algorithmicindent} where $y_{i} = r_{i} + \gamma Q^{\mathrm{tar}}_{\phi^{\prime}}(s'_{i}, \pi_{\psi^{\prime}}(s'_{i}))$ and $y_{i}' = r_{i} + r_{i}^{in} + \gamma Q^{\mathrm{beh}}_{\psi^{\prime}}(s'_{i}, \pi_{\varphi^{\prime}}(s'_{i}))$. ~~\textbf{//Update the critic}}

\STATE{\hspace{\algorithmicindent} Update $\varphi$ following the gradient
    \vspace{-0.5em}
    \begin{align*}
        \nabla_{\varphi} J_{\pi} (\varphi) := \frac{1}{M} \sum_{i = 1}^{M} \nabla_{a} Q^{\mathrm{tar}}_{\phi} (s_{i}, a) \nabla_{\varphi} \pi (s_{i}) \mid_{a = \pi (s_{i}),}
    \end{align*}
    \vspace{-0.3em}
}
\STATE{\hspace{\algorithmicindent} where $\pi$ is defined according to Eq~\eqref{eq: ADAC (DDPG) pi definition}. ~~\textbf{//Update the target policy}}

\STATE{\hspace{\algorithmicindent} Sample $\{ \xi_{j} \}_{j = 1}^{K}$, where $\forall j = 1, \dots, K, \xi_{j} \sim \mathcal{N} (\mathbf{0}, \mathbf{I})$. Update $\varphi$ following the gradient
    \begin{align*}
        \nabla_{\varphi} J_{\mu} (\varphi) := \frac{1}{M^{2}} \sum_{i = 1}^{M} \sum_{j = 1}^{M} [ \mathcal{K} (a, a'_{j}) \nabla_{a'_{j}} Q^{\mathrm{beh}}_{\psi} (s_{i}, a'_{j}) + \beta \cdot \nabla_{a'_{j}} \mathcal{K} (a, a'_{j}) ] \mid_{a = f_{\varphi} (s_{i}, \xi_{j})} \cdot \nabla_{\varphi} f_{\varphi} (s_{i}, \xi_{j}),
    \end{align*}
}
\STATE{\hspace{\algorithmicindent} where $\mathcal{K}(a, \hat{a}) = \frac{1}{\sqrt{2 \pi} (d / K)} \exp \left ( - \frac{\left \| a - \hat{a} \right \|^{2}}{2 (d / K)^2} \right )$. ~~\textbf{//Update the behavior policy}}

\STATE{\hspace{\algorithmicindent} $\phi^{\prime} :=\tau \phi+(1-\tau) \phi^{\prime}$; $\varphi^{\prime} :=\tau \varphi+(1-\tau) \varphi^{\prime}$; $\theta^{\prime} :=\tau \theta+(1-\tau) \theta^{\prime}$ ~~\textbf{//Update target networks}}

\STATE{\textbf{return}}

\end{algorithmic}
} 
\end{algorithm}

\begin{algorithm}[t]
\caption{Adventurous Actor-Critic (ADAC) with TD3 as the base model}
\label{alg:ADAC_TD3}
{\fontsize{9}{9} \selectfont
\begin{algorithmic}[1]

\STATE{\textbf{input:} environment $\mathcal{E}$, batch size $M$, maximum episode length $T_{\mathrm{max}}$, dimension of the action space $d$, mini-batch size $K$, policy update interval $d_{\pi}$, $c$, $\sigma$, $n_{\xi}$, and $\tau$.} 

\STATE{\textbf{initialize:} networks $Q^{\mathrm{tar}}_{\phi_{k}}$, $Q^{\mathrm{beh}}_{\psi_{k}}$, and $f_{\varphi}$; target networks $Q^{\mathrm{tar}}_{\phi^{\prime}_{k}}$, $Q^{\mathrm{beh}}_{\psi^{\prime}_{k}}$, and $f_{\varphi^{\prime}}$ ($k = 1, 2$); replay buffer $\mathbb{B}$. $Q^{\mathrm{tar}}_{\phi}$ and $Q^{\mathrm{beh}}_{\phi}$ correspond to $Q^{\pi}_{\rewards}$ and $Q^{\mu}_{\rewards'}$ in the main text, respectively. $\phi^{\prime}_{k} := \phi_{k}$; $\psi^{\prime}_{k} := \psi_{k}$; $\varphi^{\prime} := \varphi$ ($k = 1, 2$).}

\STATE{Define the deterministic target policy $\pi$ and the stochastic behavior policy $\mu$ with $f_{\varphi}$:
    \vspace{-0.3em}
    \begin{gather}
        \pi (s) := f_{\varphi} (s, \xi) \mid_{\xi = [0,0,\dots,0]^{T}}, \label{eq: ADAC (TD3) pi definition} \\
        \mu (\cdot \mid s) := f_{\varphi} (s, \xi) \mid_{\xi \sim \mathcal{N} (\mathbf{0}, \mathbf{I})} + \mathcal{N} (\mathbf{0}, d/K), \label{eq: ADAC (TD3) mu definition}
    \end{gather}
    \vspace{-1em}
}
\STATE{where $s$ and $\xi$ are input to the neural network $f_{\varphi}$ (see Figure~\ref{fig: policy network architecture} for its structure). For the target policy, $\xi$ is fixed as a $n_{\xi}$-dimensional zero vector. To sample an action from the behavior policy, we first sample $\xi$ from $\mathcal{N} (\mathbf{0}, \mathbf{I})$, feed it into $f_{\varphi}$ together with $s$, and add Gaussian noise $\mathcal{N} (\mathbf{0}, d / K)$.}

\vspace{0.3em}

\REPEAT

\STATE{Reset the environment $\mathcal{E}$ and receive the initial state $s_{0}$.}

\FOR{$t = 0, \ldots, T_{\mathrm{max}} - 1$}

\STATE{Sample $a_{t}$ from $\mu (\cdot \mid s_{t})$ according to Eq~\eqref{eq: ADAC (TD3) mu definition}.}

\STATE{Execute $a_{t}$ in $\mathcal{E}$ and observe environment reward $r_{t}$, intrinsic reward $r_{t}^{\mathrm{in}}$ (from any intrinsic motivation approach, or simply set to zero), and the next state $s_{t+1}$.}

\STATE{Store tuple $\left(s_{t}, a_{t}, r_{t}, r_{t}^{in}, s_{t+1}\right)$ in replay buffer $\mathbb{B}$.}

\STATE{Call \textbf{training procedure}}

\STATE{\textbf{Break} if the current episode terminates on $s_{t+1}$.}

\ENDFOR

\UNTIL{$\varphi^{\prime}$ converge or reaching the pre-defined number of training steps}

\STATE{\textbf{return} $\varphi^{\prime}$}

\STATE{}

\STATE{\textbf{training procedure}}

\STATE{\hspace{\algorithmicindent} Sample a minibatch of $M$ samples $\{ ( s_{i}, a_{i}, r_{i}, r_{i}^{in}, s'_{i} ) \}_{i = 1}^{M}$ from $\mathbb{B}$.}

\STATE{\hspace{\algorithmicindent} For all $i = 1, \dots, M$, $a'_{i} \leftarrow \pi_{\varphi^{\prime}} (s'_{i}) + \epsilon \quad (\epsilon \sim \mathtt{clip} (\mathcal{N} (0, \sigma), -c, c)$}

\STATE{\hspace{\algorithmicindent} Update $\phi$ and $\psi$ by minimizing the losses ($k = 1, 2$)
    \begin{gather*}
        \mathcal{L}_{k} (\phi) = \frac{1}{M} \sum_{i = 1}^{M} (Q^{\mathrm{tar}}_{\phi_{k}}(\rlstate_{i}, \action_{i}) - y_{i})^{2}, \\
        \mathcal{L}'_{k} (\psi) = \frac{1}{M} \sum_{i = 1}^{M} (Q^{\mathrm{beh}}_{\psi_{k}}(\rlstate_{i}, \action_{i}) - y_{i}')^{2},
    \end{gather*}
}
\STATE{\hspace{\algorithmicindent} where $y_{i} = r_{i} + \gamma \min_{k = 1, 2} Q^{\mathrm{tar}}_{\phi^{\prime}_{k}}(s'_{i}, a'_{i})$ and $y_{i}' = r_{i} + r_{i}^{in} + \gamma \min_{k = 1, 2} Q^{\mathrm{beh}}_{\psi^{\prime}_{k}}(s'_{i}, a'_{i})$. ~~\textbf{//Update the critic}}

\STATE{\hspace{\algorithmicindent} \textbf{if} $t \mathrm{mod} d == d - 1$ \textbf{then}}

\STATE{\hspace{2\algorithmicindent} Update $\varphi$ following the gradient
    \vspace{-0.5em}
    \begin{align*}
        \nabla_{\varphi} J_{\pi} (\varphi) := \frac{1}{M} \sum_{i = 1}^{M} \nabla_{a} Q^{\mathrm{tar}}_{\phi_{1}} (s_{i}, a) \nabla_{\varphi} \pi (s_{i}) \mid_{a = \pi (s_{i}),}
    \end{align*}
    \vspace{-0.3em}
}
\STATE{\hspace{2\algorithmicindent} where $\pi$ is defined according to Eq~\eqref{eq: ADAC (TD3) pi definition}. ~~\textbf{//Update the target policy}}

\STATE{\hspace{2\algorithmicindent} Sample $\{ \xi_{j} \}_{j = 1}^{K}$, where $\forall j = 1, \dots, K, \xi_{j} \sim \mathcal{N} (\mathbf{0}, \mathbf{I})$. Update $\varphi$ following the gradient
    \begin{align*}
        \nabla_{\varphi} J_{\mu} (\varphi) := \frac{1}{M^{2}} \sum_{i = 1}^{M} \sum_{j = 1}^{M} [ \mathcal{K} (a, a'_{j}) \nabla_{a'_{j}} Q^{\mathrm{beh}}_{\psi_{1}} (s_{i}, a'_{j}) + \beta \cdot \nabla_{a'_{j}} \mathcal{K} (a, a'_{j}) ] \mid_{a = f_{\varphi} (s_{i}, \xi_{j})} \cdot \nabla_{\varphi} f_{\varphi} (s_{i}, \xi_{j}),
    \end{align*}
}
\STATE{\hspace{2\algorithmicindent} where $\mathcal{K}(a, \hat{a}) = \frac{1}{\sqrt{2 \pi} (d / K)} \exp \left ( - \frac{\left \| a - \hat{a} \right \|^{2}}{2 (d / K)^2} \right )$. ~~\textbf{//Update the behavior policy}}

\STATE{\hspace{2 \algorithmicindent} $\phi^{\prime} :=\tau \phi+(1-\tau) \phi^{\prime}$; $\varphi^{\prime} :=\tau \varphi+(1-\tau) \varphi^{\prime}$; $\theta^{\prime} :=\tau \theta+(1-\tau) \theta^{\prime}$ ~~\textbf{//Update target networks}}

\STATE{\hspace{\algorithmicindent} \textbf{end if}}

\STATE{\textbf{return}}

\end{algorithmic}
}
\end{algorithm}

This section provides algorithm details of the proposed algorithm \emph{Analogous Disentangled Actor-Critic} (ADAC). For the readers' convenience, we provide pseudo-code of both the DDPG-based ADAC (Algorithm~\ref{alg:ADAC_DDPG}) and the TD3-based ADAC (Algorithm~\ref{alg:ADAC_TD3}), despite their similarities. In the following, we use Algorithm~\ref{alg:ADAC_DDPG} as an example to provide a comprehensive view of the ADAC algorithm.

As outlined in Section~\ref{Algorithm Overview}, ADAC iterates between the two main procedures, i.e., \emph{sample collection} and \emph{model update}. The sample collection phase consists of lines 8 to 10 in Algorithm~\ref{alg:ADAC_DDPG}, where the behavior policy $\mu$ (defined in lines 3 and 4) is used to interact with the environment and collect samples of experience. The model update phase (lines 17 to 26) is invoked after a new sample is added to the replay buffer. It consists of three steps: \emph{critic update}, \emph{target policy update}, and \emph{behavior policy update}. Note that the target policy and the behavior policy share the same neural network $f_{\varphi}$, the latter two steps are both devoted to update the corresponding parameter $\varphi$. The critic update phase (lines 19 and 20) follows Eq~\eqref{Eq:critic_objective}, where the target critic $Q^{\mathrm{tar}}_{\phi}$ is updated with respect to the environment-defined reward $\rewards$ and the target policy $\pi$, while the behavior critic $Q^{\mathrm{beh}}_{\psi}$ is updated with regard to the augmented reward $\rewards + \rewards^{\mathrm{in}}$ and the target policy $\pi$. The target policy update step (lines 21 and 22) and the behavior policy update step (lines 23 and 24) follow the gradient defined by Eqs~\eqref{Eq:deterministic-policy-gradient} and \eqref{eq:SVGD_gradient}, respectively. Two optimizers are used to apply the gradients $\nabla_{\varphi} J_{\pi} (\varphi)$ (line 21) and $\nabla_{\varphi} J_{\mu} (\varphi)$ (line 23) to the parameters $\varphi$, respectively.



\begin{figure}[t!]
\centering
\includegraphics[width=0.55\textwidth]{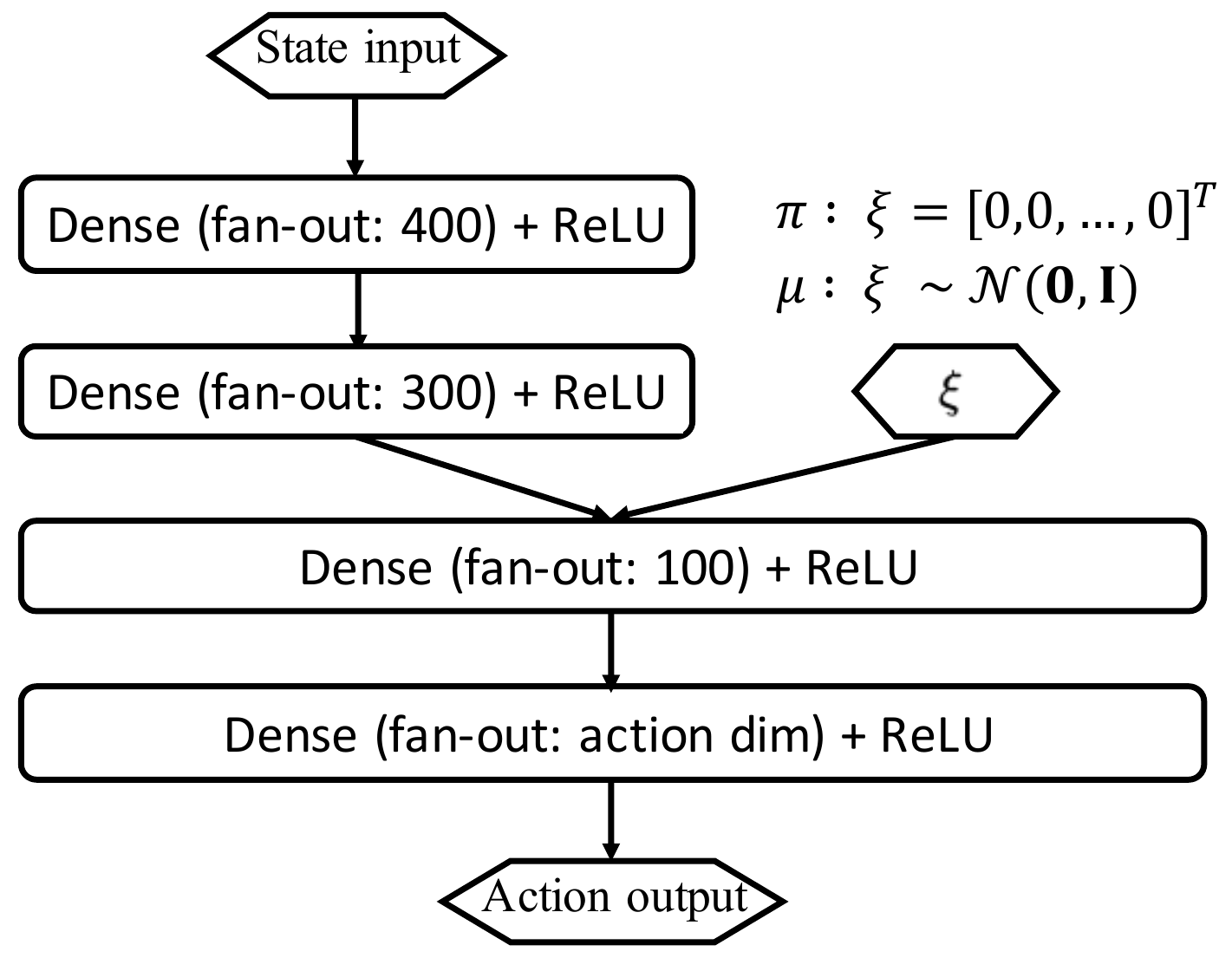}
\caption{Specifications of the policy network in ADAC.}
\label{fig: policy network architecture}
\end{figure}

\begin{table*}[t]
\caption{Environments adopted in the benchmark test. We adopted Roboschool environments with -v1 and Mujoco tasks with -v2.}
\label{table_envs}
\centering
{\fontsize{9}{9}\selectfont

\begin{tabular}{cccc}
    \toprule
    Environment & Description & $\mid \states \mid$ & $\mid \actions \mid$ \\
    \midrule \midrule
    RoboschoolAnt & Make a four-legged ant crawl & 28 & 8 \\
    RoboschoolHopper & Make a 2D robot hop & 15 & 3 \\
    RoboschoolHalfCheetah & Make a 2D cheetah run & 26 & 6 \\
    RoboschoolAtlasForwardWalk & Make the Boston Dynamics ATLAS robot run & 70 & 30 \\
    RoboschoolWalker2d & Make a two-legged robot run & 22 & 6 \\
    Ant & Make a four-legged creature walk & 28 & 8 \\
    Hopper & Make a 2D robot hop & 11 & 3 \\
    HalfCheetah & Make a 2D cheetah robot run & 17 & 6 \\
    Walker2d & Make a two-legged robot run & 15 & 3 \\
    InvertedPendulum & Balance a pole on a cart & 4 & 1 \\
    InvertedDoublePendulum & Balance a pole on a pole on a cart & 11 & 1 \\
    BipedalWalker & Move a two-legged robot on flat road & 24 & 4 \\
    BipedalWalkerHarder & Move a two-legged robot on bumpy road & 24 & 4 \\
    Lunar Lander Continuous & Navigate a lander to its landing pad & 8 & 2 \\
    \bottomrule
    
\end{tabular}
}
\end{table*}

\section{Continuous-Control Benchmarks}
\label{Continuous Control Benchmarks}

In Section~\ref{Benchmark tests} we adopted 14 standard continuous control benchmarks developed from the OpenAI Gym \cite{openaigym} package, powered by either the Mujoco physics simulator \cite{todorov2012mujoco} or the Roboschool simulator \cite{roboschool2019roboschool}. For each domain, the observation consists of physical states such as positions and angles. A brief description of each domain as well as its size of the state and action space is detailed in Table~\ref{table_envs}.

\section{sparse-reward Environments}
\label{sparse-reward environments}
In Section~\ref{Evaluation in Sparse-Reward Environments}, four sparse-reward tasks are adopted. Without otherwise noted, the prototypes of the four environments are available from the OpenAI Gym toolkit \cite{openaigym}. In the following, we provide the detailed setup of the four environments.

\vspace{0.3em}
\noindent \textbf{MountainCarContinuous} $\;$ We make no modifications to the original MountainCarContinuous. Its goal is to drive a car up the hill by applying left/right force. In addition to the goal-state reward, which is 100, the agent receives a negative reward, which equals to the magnitude of action, that is, $r (s, a) := 100 \cdot \mathbbm{1} [ s' \mathrm{is~the~goal~state}] - | a |$ ($s'$ denotes the next state after executing action $a$ in state $s$).

\vspace{0.3em}
\noindent \textbf{AcrobotContinuous} $\;$ AcrobotContinuous converts the original discrete action space of Acrobot-v1 (can be found in the OpenAI Gym Toolkit) to a continuous one, by mapping the actions uniformly to $[-1,1]$:
\begin{equation}
\nonumber
	a_{\mathrm{disc}} = \left\{\begin{array}{ccc}{\text{Original action 0}} & { a_{\mathrm{cont}} \in [-1, -1/3) } \\ {\text{Original action 1}} & { a_{\mathrm{cont}} \in [-1/3, 1/3) } \\ {\text{Original action 2}} & { a_{\mathrm{cont}} \in [1/3, 1] } \end{array}\right.,
\end{equation} 
\noindent where $a_{\mathrm{disc}}$ is the action to be applied to the original task (Acrobot-v1), and $a_{\mathrm{cont}}$ is the continuous action whose action space is $[0, 1]$.
The original Acrobot-v1 environment aims to swing up a two-linked-arm. Positive reward is received only after the end of the arm reaches above a pre-defined height.

\vspace{0.3em}
\noindent \textbf{PendulumSparse} $\;$ PendulumSparse borrows the base model from Pendulum-v0 (can be found in the OpenAI Gym Toolkit), where the agent learns to balance a single pendulum by applying clock-wise or counter clock-wise torque. We change the original dense reward to a sparse one:
\begin{equation}
    \nonumber
	r = \left\{\begin{array}{cc}{10.0} & { \cos (\theta) > 0.95 } \\ {0.0} & {\text { otherwise }}\end{array}\right.,
\end{equation} 
\noindent where $\theta$ is the pole's angle w.r.t. the verticle axis (i.e. when pointed upright, $\theta = 0$). 

\vspace{0.3em}
\noindent \textbf{CartPoleSwingUp} $\;$ The original CartPoleSwingUp task is downloaded from GitHub (\url{https://github.com/TTitcombe/CartPoleSwingUp}). It has a pole connected to a cart, which is placed on a plane. Initially, the pole points down due to the gravity. The task is to swing up the pole and balance it on top of the cart. On top of its original designs, we sparsify rewards by suppressing them unless $\cos (\theta) > 0.8$, where $\theta$ again denotes the pole's angle. An additional action penalty $-0.1 | \action |$ (the action space is 1-dimensional) is also added:
    \begin{align*}
        r' (s, a) = \left \{ \begin{array}{cc}{r (s, a) - 0.1 |a|} & {\cos (\theta) > 0.8} \\ {-0.1 |a|} & {\mathrm{otherwise}} 
        \end{array} \right.,
    \end{align*}
\noindent where $r (s, a)$ is the reward defined by the original task, and $r' (s, a)$ is the modified reward.

\section{Hyper-parameters}
\label{Hyper-parameters and Network Structure}

We have made great efforts to make sure we have a fair comparison with baselines. To be more specific, to best retain their performance, we always adopt hyper-parameters reported in the respective papers and use their open-source code when available; if not, we have to resort to our own versions, which are always based on the most-stared third-party implementations in Github. Since ADAC is based on existing off-policy models, i.e. DDPG and TD3, we fix the original hyper-parameters such as replay memory size and batch size, and tune parameters introduced by ADAC only, i.e., $K$, $\beta$, and learning rate of $f_{\varphi}$. Specifically, the entropy controlling factor $\beta$ is annealed from 2.0 to 1.0 during training, and the policy network sample size $K$ is 32 for all tasks except Hopper and RoboschoolAtlasForwardWalk, where it is 8. $\xi$ is selected as a 16-dimensional standard normal random variable (when computing behavior policy) and a 16-dimensional zero vector (when computing target policy). In ADAC (TD3), learning rate for the target policy and behavior policy is $1e-3$ and $3e-4$, respectively. For ADAC (DDPG), learning rates of both policies are set to $1e-4$.

\section{Full Benchmark Result}
\label{Full Benchmark Result}

This section presents experimental result of our proposed method ADAC along with 4 baselines on 14 continuous control benchmarks. Training curves are illustrated in Figure~\ref{fig:benchmark_learning_curves}. For final episode return, please refer to Table~\ref{table:benchmark} in the main text.

\begin{figure*}[h]
\centering
\includegraphics[width=\textwidth]{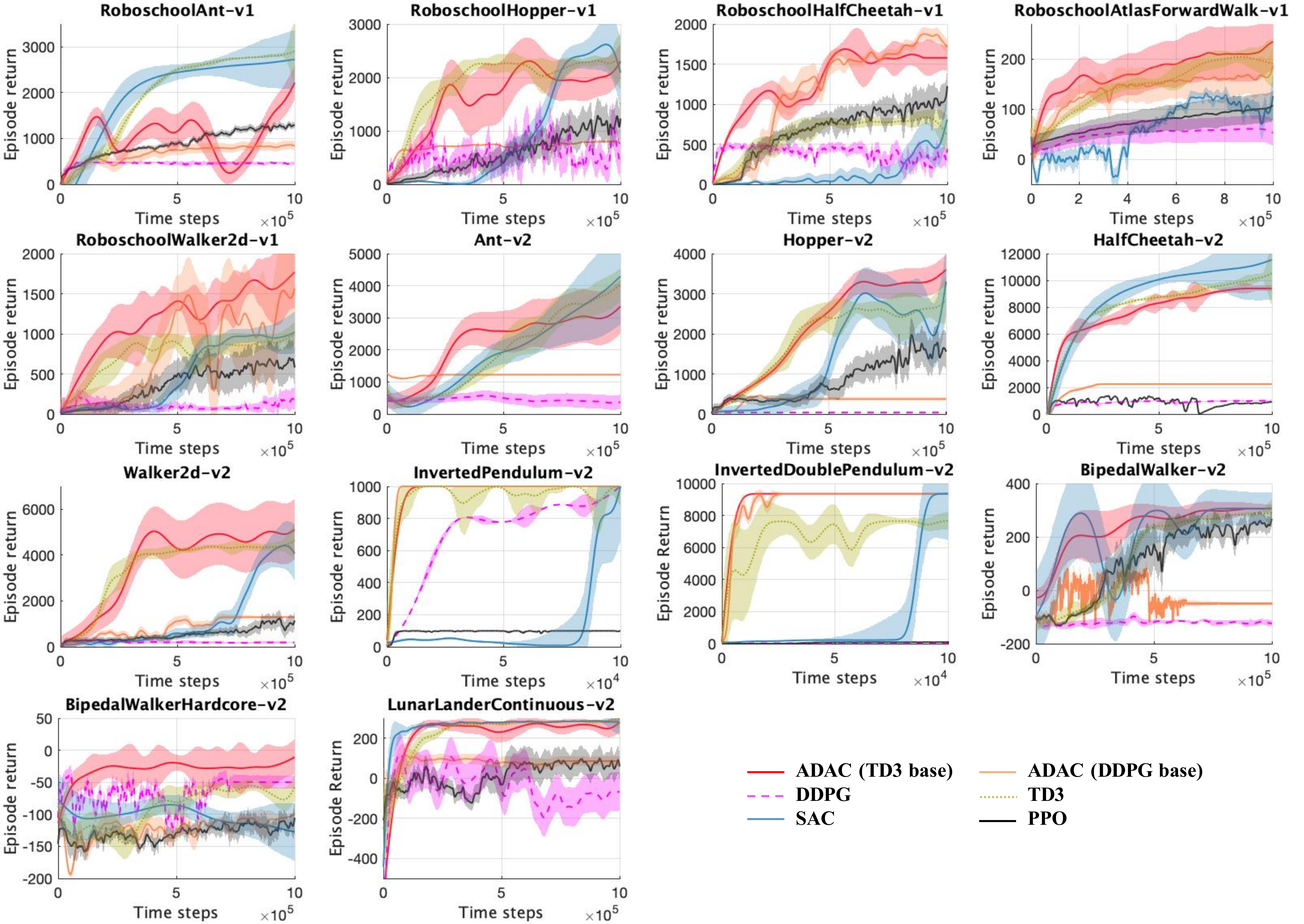}
\caption{Learning curves of our ADAC along with 4 baselines on 14 continuous-control domains smoothed over 2000 time steps. Lines denote the average over 20 trials and the shaded areas represent the range of one standard deviation. }
\label{fig:benchmark_learning_curves}
\end{figure*}

\end{document}